\documentclass{article}

\PassOptionsToPackage{numbers, compress}{natbib}

\usepackage[main, final]{neurips_2026}

\usepackage[utf8]{inputenc}
\usepackage[T1]{fontenc}
\usepackage{hyperref}
\usepackage{url}
\usepackage{booktabs}
\usepackage{amsfonts}
\usepackage{nicefrac}
\usepackage{microtype}
\usepackage[table]{xcolor}

\usepackage{amsmath}

\usepackage{multirow}
\usepackage{booktabs}

\usepackage{enumitem}

\usepackage{graphicx}
\usepackage{subcaption}
\usepackage{wrapfig}
\usepackage{algorithm}
\usepackage[noend]{algpseudocode}
\usepackage{amsthm}
\usepackage{needspace}
\usepackage{titletoc}
\newtheorem{proposition}
{Proposition}

\title{Triggering Generalist Reasoning via Predictive Uncertainty for Dual-System VLA}

\author{
Hyemin Yang\textsuperscript{1}\thanks{Equal contribution}\quad
Wooseong Jeong\textsuperscript{1}\footnotemark[1]\quad
Giwon Lee\textsuperscript{1}\footnotemark[1]\quad
Kuk-Jin Yoon\textsuperscript{1}\thanks{Corresponding author.}\\
\textsuperscript{1}\,KAIST\\[3pt]
\texttt{\{hyemin0806,stk14570,dlrldnjs,kjyoon\}@kaist.ac.kr} \\
\href{https://tudneurips2026.github.io/}{\color{magenta}{https://tudneurips2026.github.io/}}
}

\begin{document}
\maketitle
\begin{abstract}
Dual-system Vision-Language-Action (VLA) models improve real-time robotic control by pairing a slow, reasoning-capable generalist with a fast specialist action expert. However, existing methods invoke the generalist at a fixed frequency, ignoring the fact that decision-making complexity varies throughout a rollout. This static strategy wastes computation in easy phases and can delay renewed reasoning when the scene changes unexpectedly. We propose \textbf{TUD} (\textbf{T}riggering generalist reasoning via predictive \textbf{U}ncertainty for \textbf{D}ual-system VLA), an adaptive inference framework that selectively skips unnecessary generalist calls. TUD measures the cross-step dispersion of action re-predictions at the upcoming chunk slot under the cached generalist context, as a predictive uncertainty signal. This signal captures how much the future action plan shifts as new observations arrive and is computed from forwards the architecture already runs, requiring neither manual phase labels nor an auxiliary uncertainty model. On VLA-Arena, it achieves a higher success rate at matched call budgets than alternative uncertainty baselines while maintaining low wall-clock overhead, and more consistently separates successful from failed rollouts.
Also, TUD finds a more favorable cost--success trade-off than non-adaptive baselines, tracing an entire operating curve as a single threshold is varied, and substantially reduces VLM calls at matched success rate. The same trade-off appears in our real-robot experiments, where TUD cuts generalist calls by 75\% relative to the strongest fixed-interval baseline while achieving an even higher success rate. Our results suggest that predictive uncertainty provides a practical criterion for adaptive reasoning in efficient VLA control.

\end{abstract}

\section{Introduction}
\label{sec:introduction}
Enabling robots to perform diverse tasks across varying environments has been a long-standing challenge. Conventional robot learning has largely focused on policies tailored to specific tasks and robot platforms, which achieve high accuracy within their target scenario but generalize poorly to novel instructions or scenes~\cite{mees2022calvin, VLA_Benchmark_zhang2024vlabench}. Recent work addresses this limitation by leveraging pretrained large language or vision-language models~\cite{PaliGemma_beyer2024paligemma, Prismatic_karamcheti2024prismatic} to reformulate manipulation as a Vision-Language-Action (VLA) problem~\cite{VLA_Archi_kim2024openvla, VLA_not_black2024pi0visionlanguageactionflowmodel, kim2025fine, bu2025univla, shukor1844smolvla}, enabling broader generalization across instructions and scenes.

The deployment of single-system VLAs is, however, bottlenecked by limited
inference throughput, as billion-scale VLMs are too slow for real-time
contact-rich manipulation. To reconcile reasoning with fast control, a body of work adopts dual-system architectures~\cite{VLA_Dual_bu2024towards, VLA_Dual_zhang2024hirt, VLA_Dual_shentu2024llms, VLA_Dual_cui2025openhelix, VLA_Dual_chen2025fast, VLA_Dual_zhouchatvla}, inspired by the separation between slow deliberative reasoning and fast reactive behavior~\cite{kahneman2011thinking}. The \emph{generalist} (System 2) is a large VLM that consumes the full multimodal context and emits high-level reasoning or a long-horizon action plan. The \emph{specialist} (System 1) is a lightweight action expert that generates low-latency action chunks conditioned on the generalist's plan. This decoupling amortizes the heavy VLM forward across many specialist steps while allowing the specialist to absorb fresh observations between generalist calls. This gain comes with a new design question that does not exist in single-system policies: \emph{when} should the
generalist be re-invoked?

Existing dual-system VLAs leave this question unanswered and use a fixed call interval that interleaves generalist call and specialist at a static frequency~\cite{VLA_Dual_bu2024towards, VLA_Dual_chen2025fast, VLA_Dual_cui2025openhelix}. OneTwoVLA~\cite{VLA_Dual_lin2025onetwovla} instead learns a decision token to toggle between reasoning and acting modes, but requires manually annotated phase boundaries, which limits scalability and transfer to unseen tasks. In reality, scene difficulty changes step by step within a single rollout, so any fixed call interval wastes generalist calls on easy phases (e.g., a static scene) and reacts too slowly when an unexpected event (e.g., a moved obstacle) arises.

An effective re-invocation signal should capture how the long-horizon action plan shifts in response to changes in the surrounding environment. When the scene evolves away from what the cached plan assumed, the predicted actions for upcoming time steps diverge from the existing plan, and a useful trigger must reflect this divergence.
To preserve the throughput advantage that motivated the dual-system split in the first place, the signal should also be computable from the specialist forwards the architecture already runs at inference time, without auxiliary uncertainty models or manually labeled phase boundaries.

We propose \textbf{TUD} (\textbf{T}riggering generalist reasoning via predictive \textbf{U}ncertainty for \textbf{D}ual-system VLA), which adapts generalist
invocation using \emph{predictive action-plan dispersion}. At each step, the specialist re-predicts the upcoming action chunk against the most recent generalist context. The dispersion is the standard deviation of these re-predictions, remaining low when the cached plan still fits the current scene and increasing when it becomes stale. The invocation rule then defers the generalist call at each decision point when the running mean of the dispersion stays below a threshold, ensuring that the cached plan does not drift too far from a fresh specialist re-prediction.

We adapt the single-system $\pi_0$~\cite{VLA_not_black2024pi0visionlanguageactionflowmodel} backbone into a dual-system. On VLA-Arena~\cite{zhang2025vla}, the dispersion provides a more effective invocation signal for adaptive generalist calling than alternative uncertainty signals~\cite{Uncertainty_Dagger_lee2025diff-dagger,kou2024bayesdiff, de2025diffusion, jazbec2025generative, Uncertainty_VLA_ACC_Liang_2026_CVPR}, achieving the highest success rate with low wall-clock under the matched budgets. It also consistently separates successful from failed rollouts and tracks within-rollout difficulty. Empirically, TUD finds a more favorable cost--success trade-off than non-adaptive baselines as the threshold is varied. Against the $\pi_0$ baseline, TUD substantially reduces VLM calls at a matched success rate.
Furthermore, we validate TUD on a real SO-101 robotic arm and show that, averaged across three tasks, TUD achieves a higher success rate with fewer generalist calls than all baselines.

Our contributions are as follows:
\vspace{-5pt}
\begin{itemize}[leftmargin=*, topsep=2pt, itemsep=1pt, parsep=0pt, partopsep=0pt]
\item We surface adaptive generalist calling as a central design axis of dual-system VLAs and identify two preconditions missing from existing fixed-interval designs, an interval-flexible architecture and a cheap online invocation signal.
\item We propose TUD, a framework that triggers the generalist via the predictive action-plan dispersion, an inference-time signal computed from existing specialist forwards. A simple thresholding rule turns this signal into a call-or-skip decision, and varying the threshold traces an entire operating curve in the (task success rate, VLM-call budget) plane.
\item TUD finds a more favorable cost--success trade-off than non-adaptive baselines, substantially reducing VLM calls at matched success rate, and the result transfers to a real SO-101 arm.
\end{itemize}

\section{Problem Formulation}
\label{sec:preliminaries}

\subsection{Dual-System VLA}
\label{sec:dual_system_formulation}

A Vision-Language-Action (VLA) policy maps a language instruction $I$ and an observation $O_t$ at step $t$ to a robot action $a_t \in \mathbb{R}^{d_a}$, where $d_a$ is the action dimension. To balance reasoning capacity with control throughput, dual-system VLAs~\cite{VLA_Dual_bu2024towards, VLA_Dual_zhang2024hirt, VLA_Dual_shentu2024llms, VLA_Dual_cui2025openhelix, VLA_Dual_chen2025fast, VLA_Dual_zhouchatvla, VLA_Dual_lin2025onetwovla} decompose this mapping into two modules. A \emph{generalist} $\mathcal{G}$ (System~2) is a large vision-language model that performs language-grounded reasoning. A \emph{specialist} $\mathcal{S}$ (System~1) is a lightweight action policy that produces low-latency actions. The generalist is invoked at a low frequency, and the specialist runs at every control step and produces actions by combining the generalist's most recent output with the latest observation.

\begin{figure}[t]
\centering
\begin{minipage}[t]{0.54\linewidth}
\vspace{0pt}
\centering
\includegraphics[width=\linewidth]
{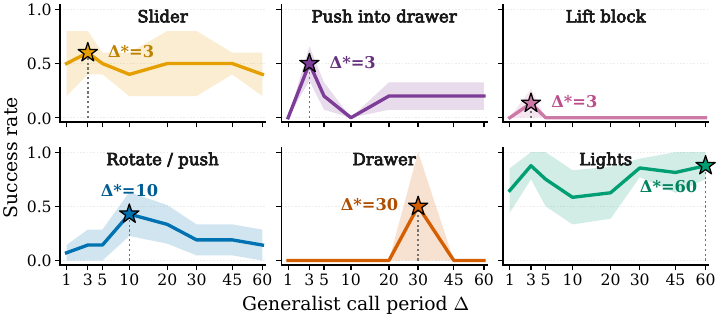}
\subcaption{OpenHelix on CALVIN: best $\Delta^\star$ varies across tasks.}
\label{fig:openhelix_freq}
\end{minipage}
\hfill
\begin{minipage}[t]{0.45\linewidth}
\vspace{-0.5mm}
\vspace{0pt}
\centering
\setlength{\tabcolsep}{4pt}
\renewcommand{\arraystretch}{1.09}
\resizebox{\linewidth}{!}{%
\begin{tabular}{ccccccc}
\toprule
\multirow{2}{*}{$\Delta$}
 & \multicolumn{5}{c}{Tasks completed in a row (\%) $\uparrow$}
 & \multirow{2}{*}{Avg.\ Len.\ $\uparrow$} \\
\cmidrule(lr){2-6}
 & 1 & 2 & 3 & 4 & 5 & \\
\midrule
2           & 96.0 & 82.0 & 71.0 & 64.0 & 51.0 & 3.64 \\
4           & 92.0 & 79.0 & 65.0 & 60.0 & 53.0 & 3.49 \\
8           & 95.0 & 74.0 & 58.0 & 51.0 & 44.0 & 3.22 \\
12          & 69.0 & 40.0 & 15.0 &  4.0 &  0.0 & 1.28 \\
16          & 59.0 & 28.0 & 12.0 &  3.0 &  0.0 & 1.02 \\
20          & 47.0 & 20.0 &  8.0 &  0.0 &  0.0 & 0.75 \\
\bottomrule
\end{tabular}%
}
\subcaption{Robodual on CALVIN at varying $\Delta$.}
\label{fig:robodual_freq}
\end{minipage}
\caption{Two properties of existing dual-system VLAs that motivate adaptive generalist calling. \textbf{(a)} Per-task success rate of OpenHelix on CALVIN at varying generalist call interval $\Delta$ ($N$ on the x-axis); the best $\Delta^\star$ differs by task category. \textbf{(b)} Long-horizon rollout performance of Robodual on CALVIN at varying $\Delta$, averaged over 100 rollouts per $\Delta$.}
\label{fig:dual_system_properties}
\end{figure}

Let $t_k$ index the time steps at which the generalist is invoked, with $t_0 < t_1 < \cdots$ and call interval $\Delta = t_{k+1} - t_k$. Existing dual-system VLAs commit to a single fixed $\Delta$ chosen offline, and two observations show why no choice is globally satisfactory. \emph{First, the optimal $\Delta$ is task-dependent.} Figure~\ref{fig:openhelix_freq} sweeps the generalist call interval for OpenHelix~\cite{VLA_Dual_cui2025openhelix} on CALVIN~\cite{mees2022calvin}: contact-rich short-horizon tasks (Slider, Push, Lift) peak at $\Delta^\star = 3$, while quasi-static toggling tasks (Lights) peak at $\Delta^\star = 60$. No fixed $\Delta$ matches both ends of this spread simultaneously.

\begin{wraptable}{r}{0.40\linewidth}
\vspace{-12pt}
\centering
\caption{Per-call cost on Robodual.}
\label{tab:throughput_latency}
\vspace{-4pt}
\small
\setlength{\tabcolsep}{4pt}
\renewcommand{\arraystretch}{1.05}
\begin{tabular}{lccc}
\toprule
 & Spec. & Gen. & Dual \\
\midrule
Throughput (Hz) $\uparrow$ & 36.98 & 2.70  & 14.79 \\
Latency (s) $\downarrow$   &   0.216 & 0.370 &  - \\
\bottomrule
\end{tabular}
\vspace{-8pt}
\end{wraptable}
\emph{Second, smaller $\Delta$ uniformly improves Robodual's~\cite{VLA_Dual_bu2024towards} long-horizon success on CALVIN.} Figure~\ref{fig:robodual_freq} shows that increasing $\Delta$ beyond $8$ collapses long-horizon success, while denser calls offer only modest gains. But denser calls erode throughput: Table~\ref{tab:throughput_latency} reports a specialist throughput of $36.98$ Hz but only $2.70$ Hz for the generalist, with each generalist call taking $0.370$ s. Thus, each extra generalist call introduces a substantial latency burden despite the benefit of denser reasoning. The right $\Delta$ also varies across rollouts of the same task, since easy and hard phases come and go within an episode. Together, these observations surface two preconditions for adaptive calling, an interval-flexible architecture that remains effective across a range of $\Delta$ values, and an inference-time signal of when a fresh generalist call is needed.

\subsection{Adaptive Generalist Calling}
\label{sec:adaptive_calling}

Existing dual-system VLAs fix the generalist call interval $\Delta = t_{k+1} - t_k$ to a single value chosen offline, which over-invokes $\mathcal{G}$ in easy phases and under-invokes it in hard ones.

We instead formulate generalist calling as a periodic call/skip decision. At each \emph{decision point}, a time step at which the rule fires, an \emph{invocation policy} $\pi$ selects an action from $\{\textsc{call}, \textsc{skip}\}$ based on the running context, choosing whether to invoke $\mathcal{G}$ at this point or to defer the call to the next one. The objective is to find $\pi^\star$ that maximizes task success while minimizing the number of generalist invocations:
\begin{equation}
    \pi^\star \;=\; \arg\max_{\pi}\; \mathbb{E}\!\left[\, \mathrm{SR}(\pi) \,\right] \;-\; \lambda \cdot \mathbb{E}\!\left[\, N_{\mathcal{G}}(\pi) \,\right],
\end{equation}
where $\mathrm{SR}(\pi)$ denotes task success rate, $N_{\mathcal{G}}(\pi)$ the number of generalist calls per rollout, and $\lambda$ trades efficiency against accuracy. Casting calling as an explicit decision problem subsumes the static-frequency setting (where $\pi$ deterministically calls every $\Delta$ steps) and exposes the headroom available to scene-conditioned invocation policies.

Realizing $\pi^\star$ requires meeting two preconditions surfaced in Section~\ref{sec:dual_system_formulation}. The architecture must remain interval-flexible so that adaptive calling does not inherit the collapse, and the invocation signal must be computable online from quantities already produced at inference time so that the trigger itself does not erode the throughput advantage of the dual-system design.

\section{Method}
\label{sec:method}

\subsection{Architecture}
\label{sec:method_arch}

\begin{figure}[t]
\centering
 \includegraphics[width=0.99\linewidth, ]{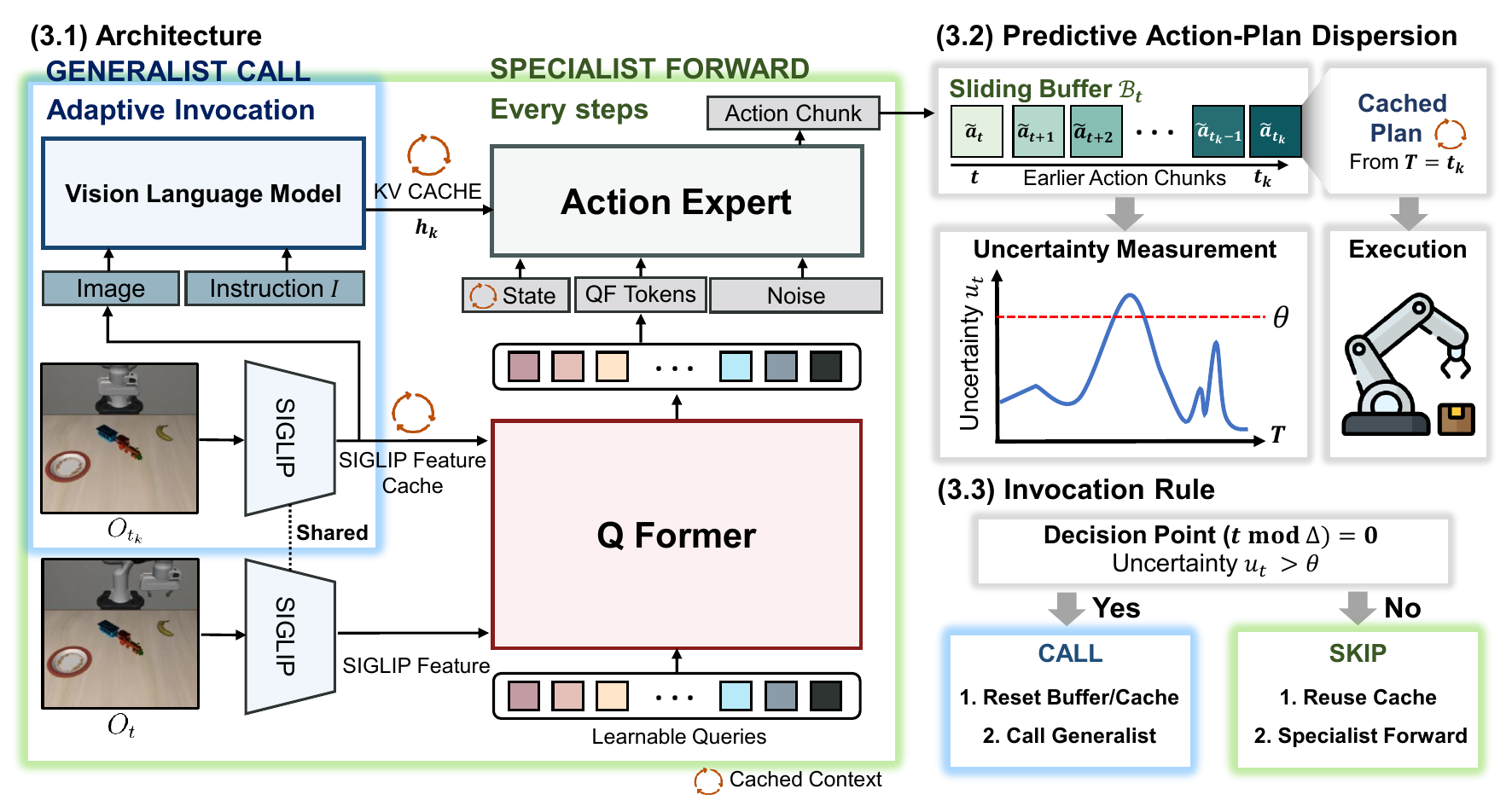}
\caption{\textbf{Overview of TUD.} The generalist call refreshes the cached context $h_k$ at decision points, the specialist forward re-predicts an action chunk against $h_k$ at every step, and a sliding buffer of these re-predictions yields the predictive action-plan dispersion $u_t$. The invocation rule thresholds the running mean of $u_t$ to skip or invoke the generalist.}
\label{fig:tud_framework}
\end{figure}

Section~\ref{sec:dual_system_formulation} surfaces two preconditions for adaptive calling, an interval-flexible architecture and a cheap online staleness signal. We meet the first with a dual-system architecture that inherits its individual modules from $\pi_0$~\cite{VLA_not_black2024pi0visionlanguageactionflowmodel}. $\pi_0$ is itself a single-system: it recomputes every stage at each control step, so the operation of the dual-system is not part of its native design. We insert a Q-Former~\cite{li2023blip} bridge, a small set of learnable query tokens that cross-attend to current and cached visual features, between the vision-language module and the action expert. The bridge lets the action expert react to the current visual frame while reading the cached language-image tokens of an earlier boundary, and this split is what gives the architecture its dual-system structure.

The observation at step $t$ is $O_t$, comprising a current visual frame and a proprioceptive state $s_t$. The dual-system split assigns the heavy Vision-Language Module (VLM, PaliGemma~\cite{PaliGemma_beyer2024paligemma} backbone) to the generalist and the Q-Former bridge plus the flow-matching action expert to the specialist. The two operational modes correspond to two distinct calls (Figure~\ref{fig:tud_framework}).

\textbf{Generalist call.} At a decision point $t_k$, the generalist takes the instruction $I$ and the current observation $O_{t_k}$, encodes the visual frame with SigLIP~\cite{SigLIP_zhai2023sigmoid}, embeds proprioception, language, and image as PaliGemma prefix tokens, and runs the LLM forward to produce a language-image key-value cache. The LLM cache together with the SigLIP feature constitute the cached context
\begin{equation}
    h_k \;=\; \mathcal{G}(I,\, O_{t_k}),
    \label{eq:generalist_call}
\end{equation}
the only state passed forward across calls. This is the heavy stage of the architecture and runs only at decision points.

\textbf{Specialist forward.} For any $t \in [t_k, t_{k+1})$, the specialist re-encodes the latest visual frame with SigLIP, runs the Q-Former bridge against the current and the cached SigLIP features to produce a fixed number of conditioning tokens, and runs the flow-matching action expert on the proprioceptive state $s_t$, the bridge tokens, and a flow-matching noise vector while reading the cached language-image tokens through cross-attention,
\begin{equation}
    \tilde a_t \;=\; \mathcal{S}(O_t,\, h_k) \;\in\; \mathbb{R}^{C \times d_a}.
    \label{eq:specialist_repred}
\end{equation}
We use the tilde to mark $\tilde a_t$ as a chunk-level prediction (the full $C$-slot chunk that the specialist outputs at step $t$), distinct from the single-step robot action $a_t \in \mathbb{R}^{d_a}$ of Section~\ref{sec:dual_system_formulation}. None of the language-image VLM tokens are recomputed.

The action chunk committed to the actuator at boundary $t_k$ is $\tilde a_{t_k}$, the first specialist forward against the refreshed cache. The actuator advances through this chunk one slot per step until the next refresh. At intermediate steps $t > t_k$, the specialist re-predicts the action chunk from the latest observation and cached context. These re-predictions are collected to measure how the predicted action plan evolves as new observations arrive and are used to compute the predictive action-plan dispersion in Section~\ref{sec:method_trigger}. By freezing the cached language-image encoding while letting the bridge and action decoder respond to the current visual observation, a single $h_k$ remains serviceable across a range of intervals rather than being pinned to a single trained $\Delta$. Architectural details of the bridge, the flow-matching schedule, and the training recipe are in Appendix~\ref{appendix:arch}.

\subsection{Predictive Action-Plan Dispersion}
\label{sec:method_trigger}

To decide whether $h_k$ has gone stale at a decision point, we examine how the specialist's re-predictions evolve from step to step. Intuitively, when the cached anchor still describes the current scene, successive re-predictions agree on the action assigned to a given future slot of the chunk. When scene drift has invalidated the anchor, that slot is re-planned differently each step. Our signal turns this stability into a single scalar.

We maintain a sliding buffer of the last $E$ specialist re-predictions
\begin{equation}
    \mathcal{B}_t \;=\; \big\{\, \tilde a_{t-i} \;:\; i = 0, \dots, E-1 \,\big\}
    \;\subset\; \mathbb{R}^{C \times d_a},
\end{equation}
reset whenever a generalist call refreshes $h_k$. Each $\tilde a_{t-i}$ is the specialist forward of Eq.~\ref{eq:specialist_repred} recorded into $\mathcal{B}_t$. We compute these re-predictions at a sparse interval $\Delta_s = 2$, with the corresponding per-call latency reported in Section~\ref{sec:exp_efficiency}.

Let $j(t) = \min(t - t_k,\, C - 1)$ denote the slot of the cached chunk that step $t$ corresponds to. Each chunk in the buffer is anchored at the same generalist call $t_k$, so reading slot $j(t)$ from every buffer entry yields $E$ alternative predictions of the action at that upcoming slot, made at different recent times $t-i$ from different observations. The predictive action-plan dispersion at $t$ averages the per-coordinate standard deviation of these buffer entries over the executed action dimensions,
\begin{equation}
    u_t \;=\; \frac{1}{d_a} \sum_{d=1}^{d_a} \mathrm{std}_{i=0,\dots,E-1}\!\Big(\, \tilde a_{t-i,\, j(t),\, d} \,\Big).
    \label{eq:uncertainty}
\end{equation}
Small $u_t$ means recent re-predictions of the upcoming chunk slots agree, so the cached plan is consistent with the current scene. Large $u_t$ means the specialist is changing its mind from step to step, indicating that the anchor no longer fits.

\subsection{Invocation Rule}
\label{sec:method_rule}

We want the generalist to be re-invoked precisely when the cached plan would diverge meaningfully from a fresh re-prediction at the upcoming slot. Under a mild Lipschitz assumption on the specialist, this deviation can be upper-bounded by two terms: a dispersion term proportional to the observable signal $u_t$, and a drift term that accounts for observation changes since the last generalist call.

\begin{proposition}[Action-deviation bound under Lipschitz specialist]
\label{prop:skip_safety}
Suppose the specialist forward $\mathcal{S}(\cdot, h_k)$ is $L_{\mathcal{S}}$-Lipschitz. Then for any $t \in [t_k, t_{k+1})$,
\begin{equation}
    \big\| \tilde a_{t_k}[j(t)] - \tilde a_t[j(t)] \big\|_2
    \;\le\;
    \alpha_E \, u_t
    \;+\;
    \beta \, L_{\mathcal{S}} \, \rho_{t,k},
    \label{eq:skip_safety}
\end{equation}
where $\alpha_E = 2 d_a \sqrt{E-1}$, $\beta \in \{0, 1\}$ depending on whether $t - t_k \le E$, and $\rho_{t,k} = \max_{t' \in [t_k, t]} \|O_{t'} - O_{t_k}\|$ is the maximal observation drift since the last generalist call.
\end{proposition}

The dispersion term $\alpha_E u_t$ reflects model uncertainty inside the buffer, while the drift term $\beta L_{\mathcal{S}} \rho_{t,k}$ accounts for observation changes once the anchor $\tilde a_{t_k}$ is no longer in the buffer. A complete proof and an empirical bound check are provided in Appendix~\ref{appendix:proof_skip_safety} and validation of $u_t$ as an invocation signal is reported in Section~\ref{sec:exp_signal}.

We instantiate the invocation policy $\pi$ of Section~\ref{sec:adaptive_calling} with a conservative call-or-skip rule that operationalizes this bound. The threshold $\theta_p$ on the running mean $\bar u$ controls the dispersion term ($\alpha_E u_t$), and the cap $\Delta_{\max}$ on the number of steps without a generalist call limits the drift term ($L_{\mathcal{S}} \rho_{t,k}$). At every decision point, the rule defers the generalist call when $\bar u$ has stayed below $\theta_p$ and the number of steps since the last generalist call has not yet reached $\Delta_{\max}$, and calls otherwise. Here $\bar u$ is the running mean of $u_t$ since the last generalist call, so each decision aggregates evidence accumulated against the current cached context rather than the instantaneous signal at the decision point. When the rule fires a \textbf{call}, the generalist refreshes $h_k$ and the actuator's committed plan $\tilde a_{t_k}$ is recomputed as the first specialist forward against the new cache; on a \textbf{skip}, both are kept and the actuator continues to advance through the existing chunk.

\textbf{Threshold selection.} We adapt $\theta_p$ online as the empirical $p$-th percentile of recently observed $u_t$ values. The percentile $p$ controls the cost--success operating point, with larger $p$ permitting more skips and fewer generalist calls. We denote settings by their percentile (e.g., \emph{p70}) and use p70 unless otherwise stated. Further details of the invocation rule, including online threshold adaptation and hyperparameter sensitivity analyses, are provided in Appendix~\ref{appendix:invo_rule}.

\section{Experiments}
\label{sec:experiments}
\subsection{Experimental Setup}
\label{sec:exp_setup}

\textbf{Backbone and training.} Our backbone is the flow-matching $\pi_0$~\cite{VLA_not_black2024pi0visionlanguageactionflowmodel} Vision-Language-Action policy with a Q-Former bridge between the vision-language module and the action expert (architecture in Appendix~\ref{appendix:arch}). We fine-tune from the released $\pi_0$ checkpoint with batch size $32$ on $4 \times$ NVIDIA H200 GPUs for $30{,}000$ steps. Action chunks have horizon $C = 50$ with $d_a = 7$ executed dimensions.

\textbf{Benchmarks and evaluation.} We evaluate on VLA-Arena~\cite{zhang2025vla}, a benchmark spanning four diverse domains (Safety, Distractor, Extrapolation, Long Horizon), comprising $11$ task suites in total, across two difficulty levels (L0, L1) that probe both nominal performance and generalization. We report L0 results in this section and L1 (generalization) results in Appendix~\ref{appendix:vla_arena_l1}. Primary metrics are Success Rate (SR) and Cumulative Cost (CC). Per-suite descriptions and the full evaluation procedure are in Appendix~\ref{appendix:benchmark_details}. Cross-benchmark validation on CALVIN ABC$\to$D~\cite{mees2022calvin} is reported in Appendix~\ref{appendix:calvin_full}.

\textbf{Real robot platform.}
We conduct real-robot experiments using a 6-DoF SO-101 single arm and two RGB cameras providing third-person and gripper views.
The task instructions used in the experiments are \textit{``Put the pen into the white bucket''}, \textit{``Close the bottom drawer''}, and \textit{``Stack the small cube on the big cube''}.
We collect 100 teleoperated demonstrations per task with varied object configurations and use them for training, and report the average performance over a total of 10 rollouts across two configurations per task at test time.
Unlike the simulation experiments, training is conducted on a single H200 GPU, while evaluation is performed on a single TITAN RTX GPU.
The full protocol, hardware, and test settings are reported in Appendix~\ref{appendix:real_robot}.

\begin{table*}[!t]
\centering
\vspace{-5pt}
\caption{Per-task success rate (SR) and cumulative cost (CC) on VLA-Arena Level 0. \emph{p70} reports TUD at our representative operating point ($p = 0.70$), and \emph{Best} reports TUD at whichever $p$ gives the highest result for each individual task.}
\vspace{-5pt}
\label{tab:vla_l0_with_ours}
\small
\setlength{\tabcolsep}{5pt}
\renewcommand{\arraystretch}{1.10}
\resizebox{\textwidth}{!}{
\begin{tabular}{lcccccccc>{\columncolor[HTML]{D5E8D4}}c>{\columncolor[HTML]{D5E8D4}}c}
\toprule
 & \multicolumn{8}{c}{\textbf{Baselines}} & \multicolumn{2}{c}{\textbf{Ours (TUD)}} \\
\cmidrule(lr){2-9} \cmidrule(lr){10-11}
\textbf{Task} & \textbf{OpenVLA} & \textbf{OpenVLA-OFT} & \boldmath$\pi_0$ & \boldmath$\pi_0$\textbf{-FAST} & \textbf{UniVLA} & \textbf{SmolVLA} & \textbf{OneTwoVLA} & \textbf{RoboDual} & \textbf{p70} & \textbf{Best}\\
\midrule
\rowcolor[HTML]{EFEFEF}
\multicolumn{11}{l}{\textbf{Safety (SR / CC)}} \\
StaticObstacles   & 0.60 / 0.0 & 1.00 / 0.0 & 0.98 / 0.0 & 1.00 / 0.0 & 0.84 / 0.0 & 0.14 / 0.0 & 0.94 / 0.0 & 0.68 / 0.0 & 0.97 / 0.0 & 0.99 / 0.0\\
CautiousGrasp     & 0.80 / 6.6 & 0.60 / 3.3 & 0.84 / 3.5 & 0.64 / 3.3 & 0.80 / 3.3 & 0.52 / 2.5 & 0.72 / 2.5 & 0.42 / 2.2 & 0.93 / 3.2 & 0.93 / 3.2\\
HazardAvoidance   & 0.20 / 17.2 & 0.36 / 9.4 & 0.74 / 6.4 & 0.10 / 10.4 & 0.70 / 5.3 & 0.16 / 10.4 & 0.56 / 6.7 & 0.68 / 5.3 & 0.85 / 5.5 & 0.85 / 5.5 \\
StatePreservation & 1.00 / 0.0 & 1.00 / 0.0 & 0.98 / 0.0 & 0.60 / 0.0 & 0.90 / 0.0 & 0.50 / 0.0 & 0.84 / 0.0 & 0.82 / 0.0 & 0.95 / 0.0 & 0.97 / 0.0\\
DynamicObstacles  & 0.60 / 3.6 & 0.80 / 8.8 & 0.92 / 6.0 & 0.80 / 3.6 & 0.26 / 7.1 & 0.32 / 2.1 & 0.76 / 6.6 & 0.32 / 3.0 & 0.89 / 7.9 & 0.93 / 7.0 \\
\midrule
\rowcolor[HTML]{EFEFEF}
\multicolumn{11}{l}{\textbf{Distractor (SR)}} \\
StaticDistractors  & 0.80 & 1.00 & 0.92 & 1.00 & 1.00 & 0.54 & 0.96 & 0.64 & 0.99 & 0.99 \\
DynamicDistractors & 0.60 & 1.00 & 0.78 & 0.80 & 0.78 & 0.42 & 0.94 & 0.52 & 0.90 & 0.97\\
\midrule
\rowcolor[HTML]{EFEFEF}
\multicolumn{11}{l}{\textbf{Extrapolation (SR)}} \\
PrepositionCombinations & 0.68 & 0.62 & 0.76 & 0.14 & 0.50 & 0.20 & 0.64 & 0.20 & 0.71 & 0.81\\
TaskWorkflows           & 0.82 & 0.74 & 0.72 & 0.24 & 0.76 & 0.32 & 0.48 & 0.42 & 0.52 & 0.66 \\
UnseenObjects           & 0.80 & 0.60 & 0.80 & 0.00 & 0.34 & 0.16 & 0.58 & 0.30 & 0.86 & 0.86 \\
\midrule
\rowcolor[HTML]{EFEFEF}
\multicolumn{11}{l}{\textbf{Long Horizon (SR)}} \\
LongHorizon & 0.80 & 0.80 & 0.92 & 0.62 & 0.66 & 0.74 & 0.90 & 0.77 & 0.97 & 0.99\\
\bottomrule
\end{tabular}
}
\vspace{-5pt}
\end{table*}

\begin{figure*}[!t]
\vspace{-0.4\baselineskip}
    \centering
    \caption{Cost--success trade-off on VLA-Arena. \textbf{Left}: Safety. \textbf{Right}: Long-Horizon. \textbf{Top}: TUD's operating curve as the percentile $p$ varies (each marker = one $p$, e.g., p70 for $p = 0.70$). \textbf{Bottom}: each non-adaptive baseline at its single operating point.}

    \begin{subfigure}[b]{0.48\textwidth}
        \centering
        \includegraphics[width=\textwidth]{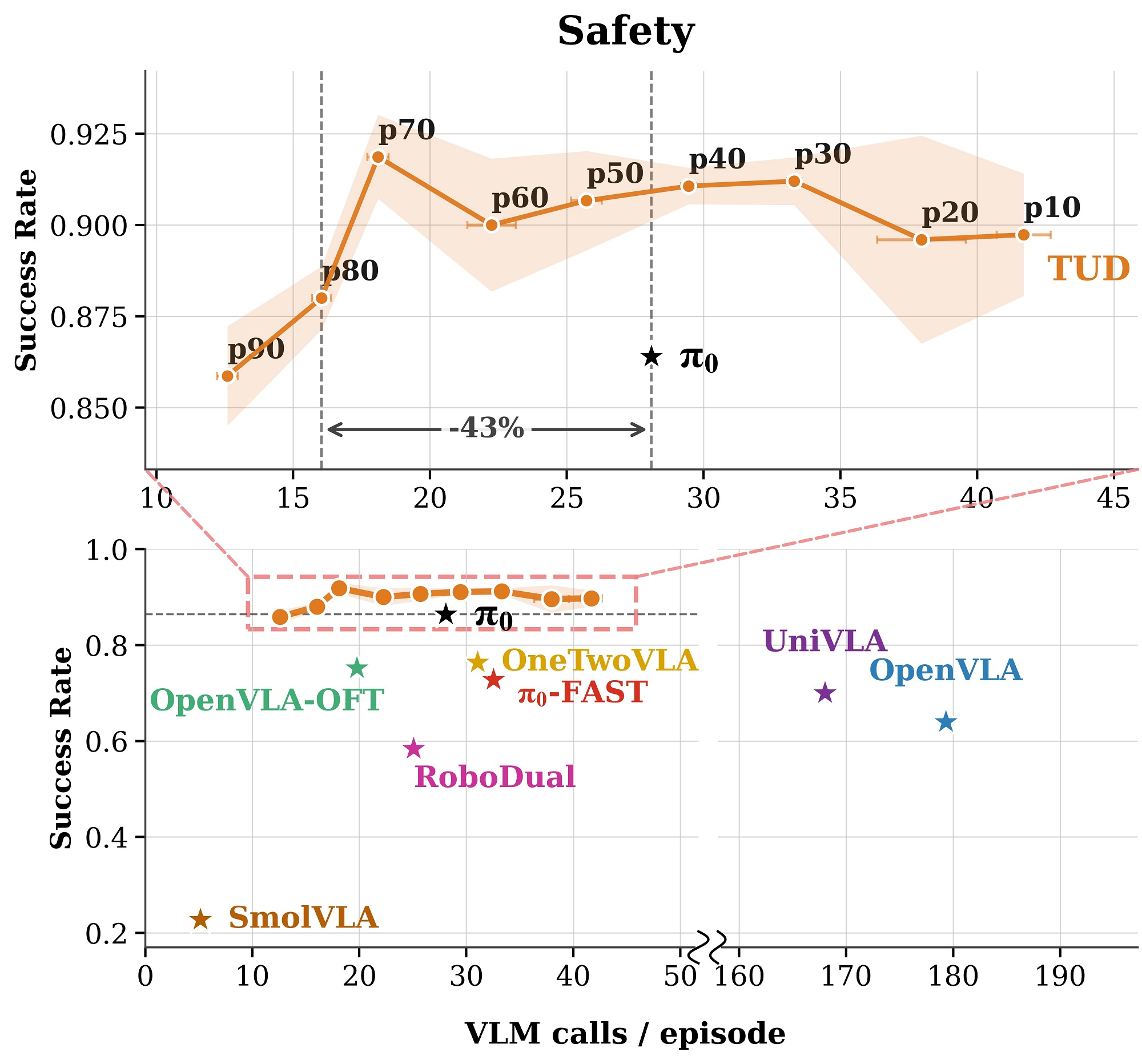}
        \label{fig:cs_safety}
    \end{subfigure}
    \begin{subfigure}[b]{0.48\textwidth}
        \centering
        \includegraphics[width=\textwidth]{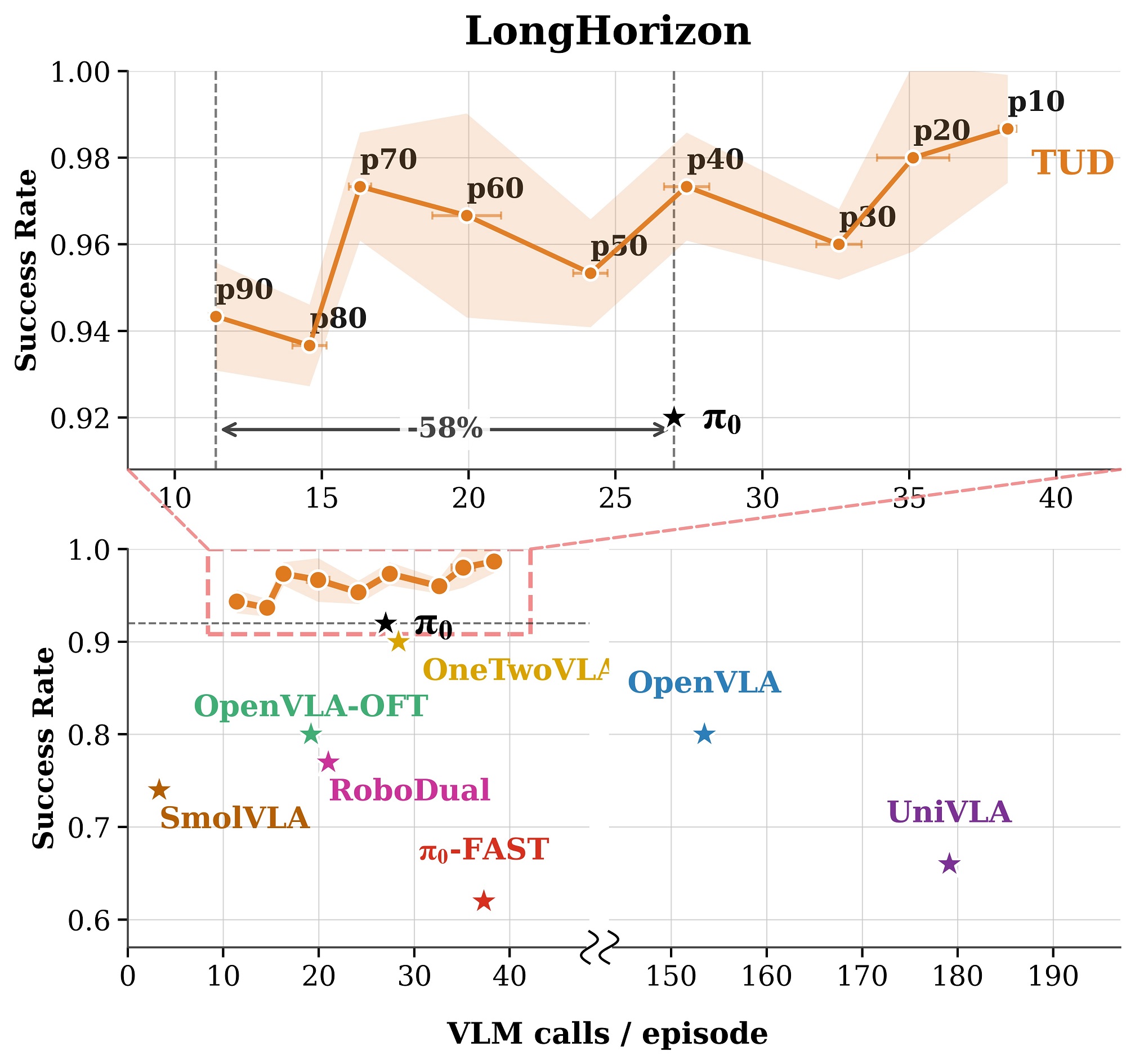}
        \label{fig:cs_longhorizon}
    \end{subfigure}
    \label{fig:cost_success}
\vspace{-0.7\baselineskip}
\end{figure*}

\subsection{Cost--Success Trade-off}
\label{sec:exp_efficiency}

TUD's contribution is a method for adapting the generalist call frequency at inference time. Different VLAs use different backbones, vision-language model sizes, and training recipes. Stronger choices on any of these axes would shift the absolute success rate higher, and we are not claiming a comprehensive SOTA comparison. We instead ask an orthogonal question: \emph{given a dual-system architecture, can we trace a better trade-off between task success rate and VLM-call budget than any non-adaptive baseline?} Existing VLAs answer only at single operating points. UniVLA~\cite{bu2025univla} and OpenVLA~\cite{VLA_Archi_kim2024openvla} invoke the VLM at nearly every step, paying a call cost that is generally unsuitable for real-time control. OpenVLA-OFT~\cite{kim2025fine}, $\pi_0$~\cite{VLA_not_black2024pi0visionlanguageactionflowmodel}, $\pi_0$-FAST~\cite{pertsch2025fast}, SmolVLA~\cite{shukor1844smolvla} and OneTwoVLA~\cite{VLA_Dual_lin2025onetwovla} amortize the call over multiple steps but each commits to a single setting with no knob for trading success rate against fewer calls. RoboDual~\cite{VLA_Dual_bu2024towards} uses a dual-system architecture but invokes the generalist on a fixed schedule. TUD exposes a tunable percentile $p \in [0, 1]$ that controls the percentile threshold of $u_t$ (Section~\ref{sec:method_rule}). Each $p$ yields one operating point, and sweeping $p$ traces an entire curve. We refer to operating points by their percentile setting, for example \emph{p70} for $p = 0.70$.

\textbf{Cost--success curve and per-task accuracy.} Figure~\ref{fig:cost_success} sweeps $p$ on the Safety and Long-Horizon domains and overlays each non-adaptive baseline at its operating point. We focus on the comparison against $\pi_0$, the source from which our backbone is initialized. On Safety, TUD at p70 matches $\pi_0$'s success rate with \textbf{43\% fewer} calls. On Long-Horizon, TUD maintains a high success rate while invoking the VLM \textbf{58\% fewer} times than $\pi_0$. Additional cost-success curves on VLA-Arena L0 are provided in Appendix~\ref{appendix:vla_arena_l0}. The per-task L0 results in Table~\ref{tab:vla_l0_with_ours} confirm that TUD at p70 stays in the competitive regime across all four domains, and the per-task best percentile further improves on every task. L1 generalization results in Appendix~\ref{appendix:vla_arena_l1} show this carries over to harder splits, and CALVIN ABC$\to$D results in Appendix~\ref{appendix:calvin_full} yield consistent findings.

\begin{wraptable}{r}{0.30\linewidth}
\vspace{-12pt}
\centering
\caption{Per-call latency.}
\vspace{-0.7\baselineskip}
\label{tab:tud_latency}
\small
\setlength{\tabcolsep}{4pt}
\renewcommand{\arraystretch}{1.05}
\begin{tabular}{lcc}
\toprule
Forward & ms $\downarrow$ & Hz $\uparrow$ \\
\midrule
Generalist & $105.91$ & $9.44$ \\
Specialist & $\phantom{0}13.53$ & $73.91$ \\
\bottomrule
\end{tabular}
\vspace{-8pt}
\end{wraptable}

\textbf{Per-call latency.} On a single H200 GPU, averaging $30$ repetitions per VLA-Arena suite ($11$ suites, $3$ seeds), a specialist forward costs roughly $1/7.8$ of a generalist call (Table~\ref{tab:tud_latency}), so each saved generalist call converts near-proportionally into wall-clock savings. The curve gain in Figure~\ref{fig:cost_success} therefore translates directly into improved real-time efficiency at matched accuracy.

\subsection{Signal Validation}
\label{sec:exp_signal}

Proposition~\ref{prop:skip_safety} ties the skip-safety bound to the observable dispersion signal $u_t$. We validate $u_t$ as an invocation signal for adaptive generalist calling from three perspectives: its effectiveness under matched call budgets, its separation between successful and failed rollouts, and its temporal dynamics within individual rollouts.

\begin{wraptable}{r}{0.45\textwidth}
\vspace{-1.0\baselineskip}
\centering
\caption{VLM re-invocation criteria at matched call budgets on VLA-Arena L0.}
\label{tab:reinvocation-criteria}

\scriptsize
\begin{tabular*}{\linewidth}{@{\extracolsep{\fill}}lccc@{}}
\toprule
\textbf{Method} &
\textbf{Calls} &
\textbf{SR (\%)} $\uparrow$ &
\textbf{Wall (s)} $\downarrow$ \\
\midrule

\multicolumn{4}{l}{\emph{High budget ($\approx$17--18 calls/ep)}} \\
\cmidrule(lr){1-4}
Diff-DAgger  & 16.5 & 83.7 & 31.1 \\
BayesDiff    & 16.3 & 83.5 & 11.9 \\
Visual       & 17.2 & 83.0 & 11.8 \\
AAC-ent.     & 16.4 & 82.7 & 28.7 \\
MC-AleaUQ    & 19.1 & 80.2 & 11.6 \\
GenUQ        & 17.3 & 74.7 & 29.7 \\
\textbf{TUD} & 17.8 & \textbf{86.7} & \textbf{9.5} \\

\midrule
\multicolumn{4}{l}{\emph{Low budget ($\approx$8--10 calls/ep)}} \\
\cmidrule(lr){1-4}
Diff-DAgger  & 8.3 & 80.7 & 31.0 \\
BayesDiff    & 6.8 & 77.0 & 11.8 \\
Visual       & 9.3 & 79.8 & 11.3 \\
AAC-ent.     & 10.1 & 77.2 & 27.7 \\
MC-AleaUQ    & 8.3 & 73.5 & 10.8 \\
GenUQ        & 8.9 & 73.8 & 30.4 \\
\textbf{TUD} & 9.0 & \textbf{82.8} & \textbf{7.9} \\

\bottomrule
\end{tabular*}
\vspace{-1.0\baselineskip}
\end{wraptable}

\textbf{Invocation signal comparison at matched call budgets.}
We first evaluate whether $u_t$ identifies when re-invocation is most useful. Under matched generalist-call budgets, all methods use the same invocation rule and differ only in the uncertainty signal driving the call-or-skip decision. We compare against a broad set of alternative criteria: rollout flow-matching loss (Diff-DAgger~\cite{Uncertainty_Dagger_lee2025diff-dagger}), Bayesian diffusion uncertainty (BayesDiff~\cite{kou2024bayesdiff}), visual-feature cosine distance (Visual), action entropy adapted from AAC~\cite{Uncertainty_VLA_ACC_Liang_2026_CVPR} (AAC-ent.), intra-step Monte-Carlo variance (MC-AleaUQ~\cite{de2025diffusion}), and posterior-ensemble dispersion (GenUQ~\cite{jazbec2025generative}). Details of each signal and its adaptation to our specialist forward $\mathcal{S}(O_t,h_k)$ are provided in Appendix~\ref{appendix:baseline_uq}. As shown in Table~\ref{tab:reinvocation-criteria}, TUD achieves the highest SR under both call-budget regimes, reaching 86.7\% at 17.8 calls and 82.8\% at 9.0 calls per episode. Moreover, $u_t$ is computed directly from specialist action predictions already available in the dual-system, requiring neither an auxiliary uncertainty model nor repeated stochastic sampling. This keeps the wall-clock overhead low compared with sampling-based alternatives. The consistent advantage at matched call budgets indicates that predictive action-plan dispersion more effectively identifies when the generalist should be invoked.

\begin{figure}[!t]
    \centering
    \vspace{-8pt}
    \includegraphics[width=\linewidth]{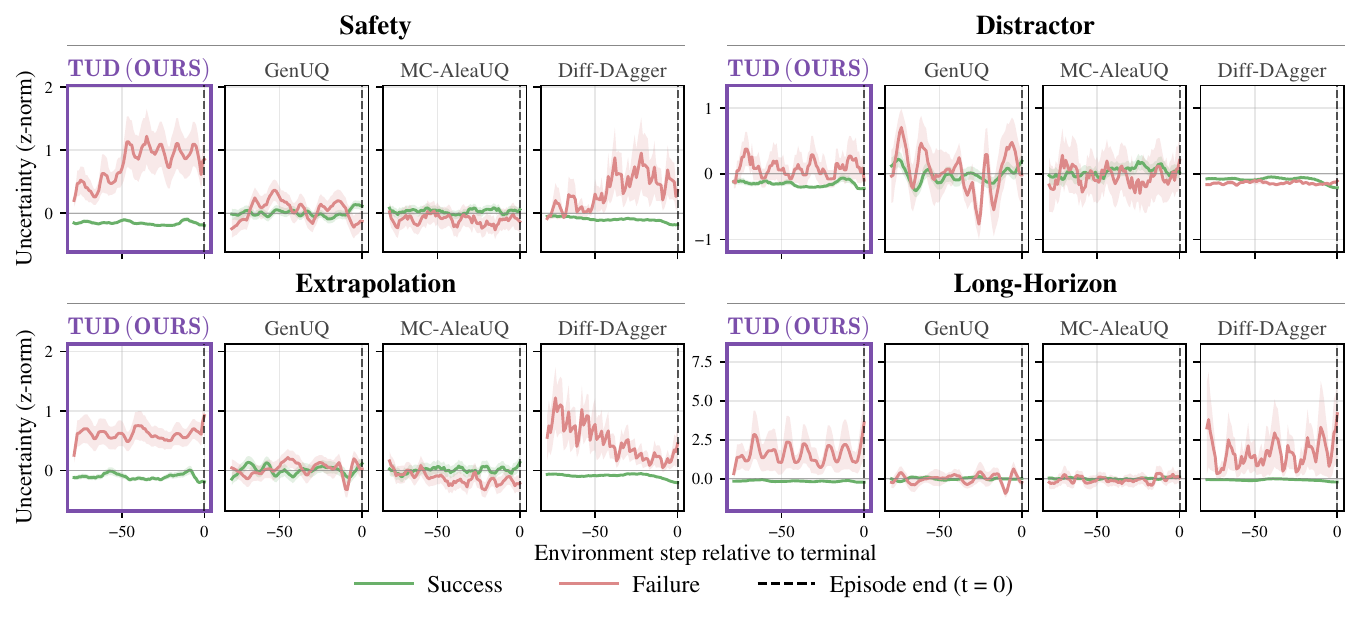}
    \caption{\textbf{Uncertainty signals on success vs failure rollouts.} For each VLA-Arena domain (one block per domain, one sub-panel per signal), we plot each signal in the final 80 steps before the rollout ends, separately for successful (green) and failed (red) episodes. Signals are $z$-normalized per method. Only $u_t$ shows a consistent gap between success and failure across all four domains.}
    \label{fig:uncertainty_compare_per_domain}
    \vspace{-0.7\baselineskip}
\end{figure}
\textbf{Failure-vs-success separation.} We next examine whether $u_t$ distinguishes rollouts that eventually succeed from those that fail. For this analysis, we focus on three representative uncertainty estimators: GenUQ, MC-AleaUQ, and Diff-DAgger. For each VLA-Arena domain, we pool L0
rollouts and compare each signal over the final episode steps for successful and failed rollouts (Figure~\ref{fig:uncertainty_compare_per_domain}). A signal whose success and failure curves separate is informative for the invocation rule, since the rule can then use it to call the generalist when the rollout is on a path toward failure. Figure~\ref{fig:uncertainty_compare_per_domain} shows that $u_t$ is the only signal that separates success from failure consistently across all four domains: the failure curve rises sharply in the final $\sim\!20$ steps while the success curve stays flat. GenUQ and MC-AleaUQ show no consistent gap in any domain, and Diff-DAgger collapses on Distractor.

\textbf{Per-rollout dynamics.}
Beyond average-trend separation across rollouts, we ask whether $u_t$ also tracks moment-to-moment difficulty within a single rollout (Figure~\ref{fig:scenario_uncertainty}). For the long-horizon task \textit{``Open the top layer of the cabinet''}, $u_t$ remains low while approaching the cabinet, rises when locating the top layer, and peaks during the handle grasping phase. This shows that $u_t$ increases at moments requiring more precise manipulation or reasoning, where a fresh generalist call is most useful. Additional per-rollout examples across other domains are provided in Appendix~\ref{appendix:uncertainty_temporal_dynamics}.

\begin{figure}[!t]
    \centering
    \vspace{-5pt}
    \includegraphics[width=0.95\linewidth]{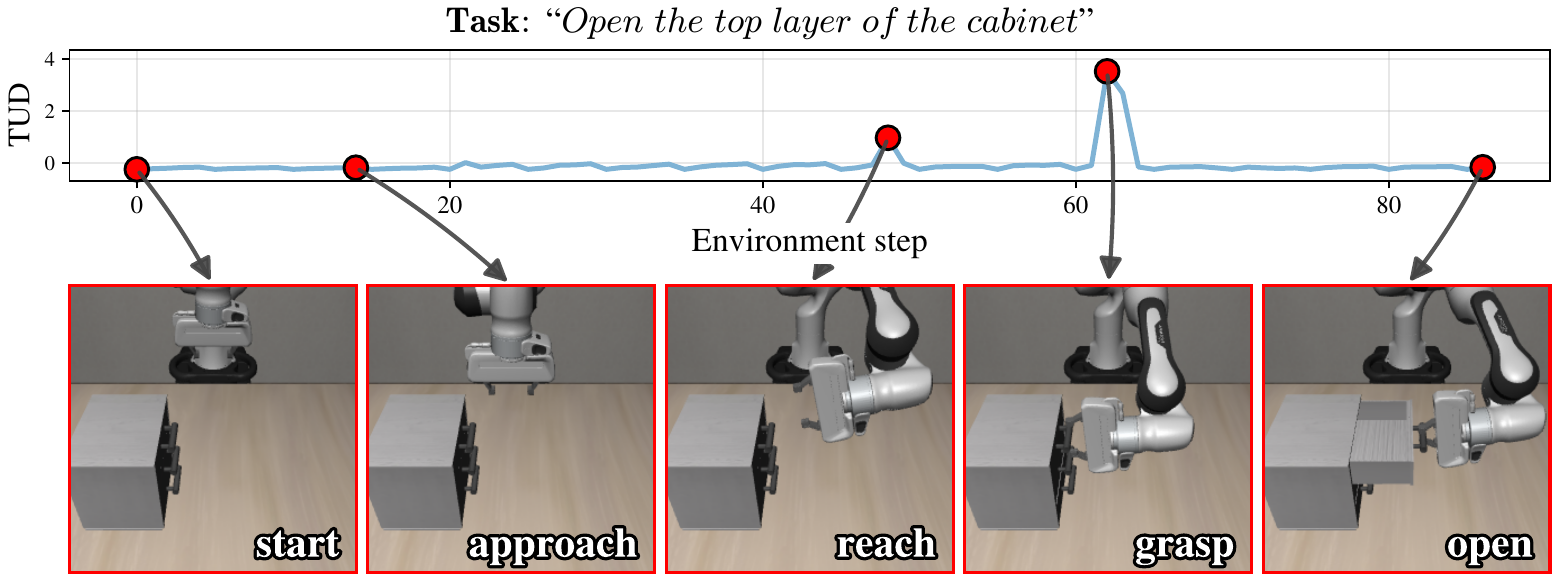}
    \caption{\textbf{Tracking $u_t$ within a representative rollout.}  Long-horizon episode for \textit{``Open the top layer of the cabinet''}. $x$-axis: environment time within the rollout. $y$-axis: $z$-normalized $u_t$ (same normalization as Figure~\ref{fig:uncertainty_compare_per_domain}). Red circles mark the highest-$u_t$ peaks plus the two rollout endpoints, with the corresponding RGB frames shown in time order below.}
    \label{fig:scenario_uncertainty}
    \vspace{-5pt}
\end{figure}

\subsection{Real Robot Experiments}
\label{sec:exp_real_robot}

\begin{wraptable}{r}{0.46\linewidth}
\vspace{-1.0\baselineskip}
\centering
\caption{\textbf{Real robot results.} SR over $30$ rollouts per method ($10$ per task). Calls and Wall (s) are per-episode generalist invocations and execution time; failures count as $120$\,s.}
\label{tab:real_robot_main_body}
\vspace{-0.6\baselineskip}
\small
\setlength{\tabcolsep}{5pt}
\renewcommand{\arraystretch}{1.05}
\begin{tabular}{lccc}
\toprule
\textbf{Method} & \textbf{SR} $\uparrow$ & \textbf{Calls} $\downarrow$ & \textbf{Wall (s)} $\downarrow$ \\
\midrule
Single-system $\pi_0$         & 0.433 & 49.0 & 79.5 \\
\midrule
Fixed $\Delta=9$      & 0.467 & 48.5 & 76.4 \\
Fixed $\Delta=25$     & 0.433 & 34.6 & 76.1 \\
Fixed $\Delta=50$     & 0.400 & 25.9 & 79.2 \\
\midrule
TUD p50          & 0.767 & 24.5 & 53.4 \\
TUD p70          & 0.733 & 20.8 & 53.5 \\
\rowcolor[HTML]{D5E8D4}
\textbf{TUD p90} & \textbf{0.833} & \textbf{12.1} & \textbf{38.5} \\
\bottomrule
\end{tabular}
\vspace{-1.0\baselineskip}
\end{wraptable}

We conduct real-robot experiments to verify whether the same favorable cost--success trade-off also appears in the real-robot setting. In Table~\ref{tab:real_robot_main_body} we compare single-system $\pi_0$ and the fixed-interval calling strategy of our $\pi_0$+Q-Former ($\Delta = 9, 25, 50$) with TUD at p50/p70/p90.
Averaged over the three tasks, all three TUD operating points achieve a higher success rate than every baseline, while issuing fewer generalist calls than any of them. Against the strongest fixed-interval baseline, $\Delta = 9$, TUD p90 raises the success rate from $0.467$ to $0.833$ while reducing generalist calls from $48.5$ to $12.1$ and wall-clock time from $76.4$\,s to $38.5$\,s. As another observation, simply tuning the fixed interval does not reach the cost--success trade-off region that TUD occupies. This is because increasing $\Delta$ from $9$ to $50$ reduces calls but also lowers the success rate from $0.467$ to $0.400$, which shows the need to decide the invocation timing adaptively. Per-task results and the end-to-end latency and control-loop jitter analyses are in Appendix~\ref{appendix:real_robot}.

\section{Conclusion}
\label{sec:conclusion}
We introduced TUD, a method for adapting the generalist call frequency at inference time in dual-system Vision-Language-Action models. TUD's signal, the predictive action-plan dispersion $u_t$, is computed entirely from existing specialist forwards with no extra training or stochastic samples, and a tunable percentile $p$ on the online-calibrated quantile of $u_t$ traces an entire operating curve in the (task success rate, VLM-call budget) plane. While existing VLAs occupy single operating points with no knob for trading accuracy against call cost, TUD finds a strictly more favorable curve on VLA-Arena and substantially reduces VLM calls at a matched success rate. The same trade-off holds on a real SO-101 arm, where the three-task average puts TUD above every baseline in success rate with fewer generalist calls. A current limitation is that $p$ controls the trade-off but does not directly optimize for a specific target operating point such as a desired latency budget or minimum success rate. Overall, our results suggest that \emph{deciding when to reason} is a key factor for practical and efficient dual-system VLA deployment.
\newpage

\clearpage
\bibliographystyle{unsrtnat}
\bibliography{main}
\clearpage
\appendix

\startcontents[supp]

\section*{Supplementary Material}

\noindent\textbf{Contents}
\vspace{-0.4em}

{\small
\setcounter{tocdepth}{2}
\printcontents[supp]{}{1}{}
}

\vspace{0.8em}

\clearpage

\section{Additional Related Work}
\label{appendix:related_work_extended}
This appendix complements the related-work pointers in the introduction with broader context across VLA architectures, dual-system designs, and uncertainty estimation for robotic policies.

\textbf{VLA models.} Beyond the architectures cited in the main paper, work on VLAs has expanded along several orthogonal axes that share TUD's emphasis on closed-loop deployability. Inference-time efficiency has been pursued through token compression and pruning~\cite{VLA_Efficient_xu2025vla, VLA_Efficient_yang2025efficientvla, VLA_Efficient_li2025cogvla, VLA_tokensmooth_zhou2025beast, li2026spvla, ni2025swiftvla}, key-value caching~\cite{TokenCompressCache_bulat2025compress, TokenCompressCache_song2025reasoning}, post-training quantization~\cite{xu2026qvla, zhang2026quantvla}, and chain-of-thought or latent reasoning that lets the model condition on intermediate plans before acting~\cite{VLA_COT_li2025coa, VLA_COT_zhao2025cot, VLA_COTandBenchmark_chen2025combatvla, VLA_longterm_COT_gu2025manualvla, bai2026latent, yang2026vlaser, zhong2026acot, peng2025counterfactual, lin2025hif}. A separate line of work fine-tunes VLAs with reinforcement learning to recover from compounding error and to satisfy task-level constraints~\cite{VLA_RL_liu2025can, VLA_RL_zhang2025safevla, VLA_RL_zhou2025autovla, VLA_RL_azzolini2025cosmos, li2026simplevlarl}. Memory, active perception, and spatial reasoning extend the standard 2D image conditioning with explicit geometric and temporal context~\cite{3d_aware_VLM_lee20253d, 3d_aware_VLA_li2025bridgevla, 3d_aware_vla_zheng2024tracevla, 3d_aware_vla_tian2025dynarend, shi2026memoryvla, liu2026activevla, xiao2025ava, feng2025spatial, xu2025affordance, qian2025geopredict, liang2026pixelvla}. Pretraining, scaling, and adaptation have explored cross-embodiment scaling~\cite{zheng2026xvla}, meta-learning for embodied tasks~\cite{li2026metavla}, action-as-language formulations that preserve VLM priors during adaptation~\cite{hancock2026actions}, latent-action pretraining from human or industrial videos~\cite{luo2026joint, zhang2025observation}, multi-skill model merging and atomic-skill composition~\cite{fu2025mergevla, zhang2026atomicvla}, and feature adaptation for viewpoint robustness~\cite{li2025vla}. Because generalist VLAs are trained on many tasks at once, they also relate to model merging, which has been used to combine motion planners trained on diverse driving datasets~\cite{lee2025interaction} and, more generally, to combine models or adapters trained on different tasks into a single model~\cite{ilharco2023editing, wortsman2022model, jeong2026preference, matena2022merging, yadav2023ties, lee2026labelfree}. Decoding-paradigm and backbone variants include unified discrete-diffusion~\cite{chen2026unified}, hybrid diffusion-autoregression~\cite{liu2026hybridvla}, video-generation backbones~\cite{shen2026videovla}, and revised vision-language formulations that revisit the role of the VLM in VLAs~\cite{zhang2026vlmvla}. Specialized deployments include autonomous driving~\cite{zhang2026cocvla, li2026drivevlaw, yuan2026autodriver, peng2025colavla, luo2026unleashing}, whole-body and bimanual manipulation~\cite{jiang2026wholebodyvla, im2026twinvla}, sim-to-real transfer~\cite{zhao2026simreal}, force-aware contact-rich control~\cite{li2026forcevla2}, and richer instruction interfaces such as speech and interleaved image-text~\cite{zhao2025vlas, fan2026interleavevla}. These directions are largely complementary to the scheduling question we study, since any of them can in principle plug into the dual-system architecture that TUD operates on.

\textbf{Dual-System VLAs.} Dual-system VLAs~\cite{VLA_Dual_zhang2024hirt} separate high-level reasoning from low-level action generation to improve control efficiency. Synchronous dual-system designs run the slow VLM in lockstep with the action expert, treating System~2's output as a per-step conditioning signal~\cite{VLA_not_black2024pi0visionlanguageactionflowmodel, VLA_Dual_yespatially, VLA_Dual_wen2025dexvla, VLA_Dual_zhouchatvla}, which cleanly preserves language grounding but inherits the VLM's per-step latency. Asynchronous variants relax this by running System~2 at a low fixed frequency while System~1 issues actions at a higher rate, and they differ mainly in what System~2 hands to System~1, ranging from conditioning tokens fed forward to a flow-matching specialist (Robodual~\cite{VLA_Dual_bu2024towards}, Fast-in-Slow~\cite{VLA_Dual_chen2025fast}), to planning tokens for a 3D-aware diffusion policy (OpenHelix~\cite{VLA_Dual_cui2025openhelix}), to latent goals refined by an action expert (Hume~\cite{VLA_Dual_song2025hume}). All of these still pin the call interval $\Delta$ at a single value chosen offline, so they sit in the design space we extend with adaptive scheduling. A single offline $\Delta$ is limiting in part because tasks differ in how often they need the generalist (Figure~\ref{fig:openhelix_freq}). This parallels multi-task learning, where tasks trained together differ in learning dynamics and gradient behavior, so treating them uniformly is suboptimal and has motivated task-adaptive loss balancing and update rules~\cite{jeong2024quantifying, chen2018gradnorm, sener2018multi, jeong2025selective, yu2020gradient, liu2021conflict, jeong2025resolving, navon2022multi}. TUD carries this view to inference time by letting the call frequency adapt to each task and to each phase of a rollout. OneTwoVLA~\cite{VLA_Dual_lin2025onetwovla} is the closest prior attempt to move beyond a fixed call interval. It learns a decision token that switches between reasoning and acting modes, supervised using manually annotated phase boundaries. This design depends on the heuristic choice of where a phase begins and may therefore limit transfer to unseen scenes.
CLOVER~\cite{bu2024closed} also re-invokes its visual planner adaptively, but the trigger is a perceptual distance against a generated subgoal image and requires training a dedicated video-diffusion subgoal generator. Adaptive Action Chunking (AAC)~\cite{Uncertainty_VLA_ACC_Liang_2026_CVPR} uses action entropy estimated from multiple sampled action chunks to adapt the length of the action chunk executed before replanning. In contrast, TUD measures predictive action-plan dispersion using forwards that the dual-system architecture already performs, and uses this signal to selectively re-invoke the generalist. This directly exposes the generalist call budget as a controllable inference-time trade-off. TUD is also distinct from test-time adaptation, which updates model parameters on unlabeled test inputs to cope with distribution shift~\cite{sun2020test, wang2021tent, jeong2025synchronizing}. TUD keeps all weights fixed and adapts only how often the generalist is invoked during a rollout.

\textbf{Uncertainty estimation for robotic policies.} Uncertainty estimation in robotic policies has been studied at both training and inference time. At training time, iterative imitation learning frameworks based on DAgger~\cite{ross2011reduction} use uncertainty to decide when to query an expert. Concrete realizations include ensemble variance over independently trained policies~\cite{Uncertainty_Dagger_menda2019ensembledagger, Uncertainty_Dagger_hoque2022thriftydagger}, MC-dropout over a single policy~\cite{Uncertainty_UAIL_cui2019uncertainty}, and diffusion-loss comparison against a training-distribution reference~\cite{Uncertainty_Dagger_lee2025diff-dagger}. At inference time, uncertainty drives failure detection~\cite{FailDetect_agia2024unpacking, Uncertainty_Diffusion_FailDetect_xu2025can, Uncertainty_FailDetect_gu2025safe} and human-in-the-loop intervention~\cite{he2025uncertainty, empleo2025safe}. These methods share a common limitation: they require auxiliary training pipelines, multiple policy models, or repeated stochastic forwards, all of which are at odds with the throughput requirements that motivate dual-system VLAs. Our work targets a signal that avoids all three by computing dispersion from forwards the architecture already runs, and Section~\ref{sec:exp_signal} empirically shows it separates success from failure on VLA-Arena more reliably than three representative prior signals from this family.

\section{Proof of Proposition~\ref{prop:skip_safety}}
\label{appendix:proof_skip_safety}

We briefly introduce the basic concepts used in the proof to keep the paper self-contained.

\noindent\textbf{Lipschitz continuity.}\ \ A map $f: \mathcal{X} \to \mathcal{Y}$ between normed spaces is $L$-Lipschitz if $\|f(x) - f(x')\| \le L \|x - x'\|$ for all $x, x' \in \mathcal{X}$. We use the Euclidean norm $\|\cdot\|_2$ on observations and action chunks throughout. The assumption is mild, and any neural network with bounded weights is Lipschitz with some finite constant.

We restate the proposition for convenience.

\noindent\textbf{Proposition~\ref{prop:skip_safety} (Action-deviation bound under Lipschitz specialist).}\ \ \emph{Suppose the specialist forward $\mathcal{S}(\cdot, h_k)$ is $L_{\mathcal{S}}$-Lipschitz. Then for any $t \in [t_k, t_{k+1})$,}
\begin{equation*}
    \big\| \tilde a_{t_k}[j(t)] - \tilde a_t[j(t)] \big\|_2
    \;\le\;
    \alpha_E \, u_t \;+\; \beta \, L_{\mathcal{S}} \, \rho_{t,k},
\end{equation*}
\emph{with $\alpha_E = 2 d_a \sqrt{E-1}$, $\beta \in \{0,1\}$ depending on whether $t - t_k \le E$, and $\rho_{t,k} = \max_{t' \in [t_k, t]} \|O_{t'} - O_{t_k}\|$.}

\begin{proof}[Proof of Proposition~\ref{prop:skip_safety}]
Because every chunk in the buffer is anchored at the same generalist call $t_k$, slot index $j(t)$ refers to the same upcoming step across all entries. For each buffer entry $i \in \{0, \dots, E-1\}$, define
\begin{equation*}
    y_i \;=\; \tilde a_{t-i}[j(t),\,:] \;\in\; \mathbb{R}^{d_a},
\end{equation*}
let $\bar y = (1/E) \sum_i y_i$ and let $s_d$ denote the per-coordinate sample standard deviation of $\{y_{i,d}\}_{i=0}^{E-1}$. By Eq.~(\ref{eq:uncertainty}) of the main paper,
\begin{equation*}
    u_t \;=\; \frac{1}{d_a} \sum_d s_d \;=\; \frac{\|s\|_1}{d_a}.
\end{equation*}

We first record a standard deviation-from-sample-mean inequality. For each coordinate $d$, the identity $\sum_{i'=0}^{E-1} (y_{i',d} - \bar y_d)^2 = (E-1) s_d^2$ implies $|y_{i,d} - \bar y_d|^2 \le (E-1) s_d^2$ for every $i$. Summing across $d$ and taking square roots,
\begin{equation}
    \|y_i - \bar y\|_2 \;\le\; \sqrt{E-1} \cdot \|s\|_2 \;\le\; \sqrt{E-1} \cdot \|s\|_1 \;=\; d_a \sqrt{E-1} \cdot u_t,
    \label{eq:proof_dev}
\end{equation}
where the second inequality uses $\|s\|_2 \le \|s\|_1$ (valid since each $s_d \ge 0$). Define
\begin{equation*}
    y_t \;:=\; \tilde a_t[j(t),\,:], \qquad y_{t_k} \;:=\; \tilde a_{t_k}[j(t),\,:].
\end{equation*}
Because the buffer is updated by inserting $\tilde a_t$ at index $0$ at every step, $y_t = y_0$ is always present in the buffer.

\textbf{Case 1: $t - t_k \le E$.} The anchor entry $\tilde a_{t_k}$ has not yet been evicted, so $y_{t_k} = y_{t - t_k}$ is also a buffer entry. Applying Eq.~(\ref{eq:proof_dev}) to both $y_t$ and $y_{t_k}$ and using the triangle inequality,
\begin{equation*}
    \|y_{t_k} - y_t\|_2 \;\le\; \|y_{t_k} - \bar y\|_2 + \|y_t - \bar y\|_2 \;\le\; 2 d_a \sqrt{E-1} \cdot u_t.
\end{equation*}
This proves the bound with $\alpha_E = 2 d_a \sqrt{E-1}$ and $\beta = 0$.

\textbf{Case 2: $t - t_k > E$.} The anchor entry has been evicted and the oldest still-in-buffer entry is $y_{E-1} = \tilde a_{t-(E-1)}[j(t),\,:]$. Both $\tilde a_{t_k}$ and $\tilde a_{t-(E-1)}$ are computed via the same cached context $h_k$, so by the Lipschitz hypothesis,
\begin{equation*}
    \|\tilde a_{t_k}[j(t),\,:] - \tilde a_{t-(E-1)}[j(t),\,:]\|_2 \;\le\; \|\tilde a_{t_k} - \tilde a_{t-(E-1)}\|_2 \;\le\; L_{\mathcal{S}} \cdot \|O_{t_k} - O_{t-(E-1)}\|_2 \;\le\; L_{\mathcal{S}} \rho_{t,k},
\end{equation*}
where the last inequality uses the definition $\rho_{t,k} = \max_{t' \in [t_k, t]} \|O_{t'} - O_{t_k}\|$ and that $t - (E-1) \in [t_k, t]$. Combining with Eq.~(\ref{eq:proof_dev}) applied to $y_{E-1}$ and $y_t$, and using triangle inequality:
\begin{align*}
    \|y_{t_k} - y_t\|_2 \;&\le\; \|y_{t_k} - y_{E-1}\|_2 + \|y_{E-1} - \bar y\|_2 + \|y_t - \bar y\|_2 \\
    \;&\le\; L_{\mathcal{S}} \rho_{t,k} + 2 d_a \sqrt{E-1} \cdot u_t.
\end{align*}
This proves the bound with $\alpha_E = 2 d_a \sqrt{E-1}$ and $\beta = 1$.
\end{proof}

The bound is loose by a factor of $\sqrt{d_a}$ in the worst case where the per-dimension dispersion is concentrated in a single coordinate. Under the more typical regime where per-dimension stds are comparable in magnitude, $\|s\|_2 \approx \sqrt{d_a}\,(\|s\|_1 / d_a)$, and the constant tightens to $\alpha_E \approx 2 \sqrt{(E-1) d_a}$.

\paragraph{Empirical bound check and interpretation.}
Proposition~\ref{prop:skip_safety} bounds the deviation between the cached specialist action and its current re-prediction at the executed action slot, with two complementary terms. The observable dispersion $u_t$ captures disagreement among recent action re-predictions, while $L_{\mathcal S}\rho_{t,k}$ accounts for changes in the observation after the cached anchor leaves the prediction buffer. The coefficient $\alpha_E=2d_a\sqrt{E-1}$ scales the dispersion term with the buffer size, and $\beta$ activates the observation-drift term only when the anchor $\tilde a_{t_k}$ is no longer contained in the buffer. We use $d_a=7$, corresponding to three translation dimensions, three rotation dimensions, and one gripper dimension.

Increasing $E$ averages the dispersion estimate over a longer prediction history, which can improve temporal stability, but simultaneously increases $\alpha_E$ and therefore does not necessarily tighten the bound. For the buffer sizes considered in our experiments,

\begin{table}[h]
\centering
\caption{Effect of the prediction-buffer size $E$ on the dispersion coefficient
$\alpha_E=2d_a\sqrt{E-1}$ with $d_a=7$.}
\label{tab:bound_buffer_coeff}
\small
\setlength{\tabcolsep}{8pt}
\begin{tabular}{ccl}
\toprule
\textbf{$E$} & \textbf{$\alpha_E$} & \textbf{Interpretation} \\
\midrule
3 & 19.80 & Tighter bound, but noisier dispersion estimate \\
5 & 28.00 & Default balance \\
7 & 34.29 & Smoother dispersion estimate, but looser bound \\
\bottomrule
\end{tabular}
\end{table}

We further evaluate how conservative the bound is on held-out test trajectories. For this diagnostic, we fix $L_{\mathcal S}$ to its calibration $p99$ value of $71.65$ and compute the bound-to-deviation ratio

\begin{equation}
r_t =
\frac{
\alpha_E u_t + \beta L_{\mathcal S}\rho_{t,k}
}{
\left\|
\tilde a_{t_k}[j(t)]-\tilde a_t[j(t)]
\right\|_2
}.
\label{eq:bound_ratio}
\end{equation}

A ratio of $1$ corresponds to a tight bound, whereas values larger than $1$ indicate an increasingly conservative margin. Table~\ref{tab:empirical_bound} reports the lower tail of this ratio for several buffer sizes.

\begin{table}[h]
\centering
\caption{Empirical bound-to-deviation ratio $r_t$ on held-out trajectories.
Larger values indicate a more conservative bound.}
\label{tab:empirical_bound}
\small
\setlength{\tabcolsep}{9pt}
\begin{tabular}{ccccc}
\toprule
\textbf{$E$} & \textbf{Minimum} & \textbf{p1} & \textbf{p5} & \textbf{p10} \\
\midrule
3  & 1.472 & 1.654 & 2.140 & 2.425 \\
5  & 2.009 & 2.210 & 2.834 & 3.220 \\
7  & 2.209 & 2.742 & 3.278 & 3.672 \\
9  & 2.324 & 3.054 & 3.624 & 4.036 \\
11 & 2.557 & 3.224 & 3.801 & 4.188 \\
\bottomrule
\end{tabular}
\end{table}

For $E\geq5$, even the minimum ratio exceeds $2$, while the p1 ratio exceeds $3$ for $E\geq9$, indicating a substantial conservative margin in the lower tail. We do not use this analysis to select an optimal buffer size. Instead, the bound provides design guidance for limiting both the prediction buffer and prolonged cached-plan reuse. We therefore use the moderate default $E=5$ together with a finite $\Delta_{\max}$, while the action-chunk horizon $C=50$ follows the base $\pi_0$ configuration.

\section{Invocation Rule Details}
\label{appendix:invo_rule}
\textbf{Implementation details.}
Algorithm~\ref{alg:skip_rule} summarizes the complete invocation procedure.
Unless otherwise stated, we use a prediction-buffer size of $E=5$, a
side-channel measurement period of $\Delta_s=2$, a decision period of
$\Delta=3$, and a maximum interval without a generalist call of
$\Delta_{\max}=20$. The action-chunk horizon is fixed to $C=50$ following
the base $\pi_0$ configuration and is not independently tuned.
\begin{algorithm}[h]
\caption{Invocation rule for adaptive generalist calling.}
\label{alg:skip_rule}
\begin{algorithmic}[1]
\Require call interval $\Delta$, threshold $\theta_p$, max steps without a call $\Delta_{\max}$, buffer size $E$
\State $t_k \gets 0$, $h_0 \gets \mathcal{G}(I, O_0)$ (Eq.~\ref{eq:generalist_call}), compute $\tilde a_0$ via Eq.~\ref{eq:specialist_repred}, reset buffer $\mathcal{B}$
\For{each step $t = 1, 2, \dots$}
    \State execute the action drawn from $\tilde a_{t_k}$ at slot $j(t)$
    \State \textbf{if} $t \bmod \Delta_s = 0$, compute $\tilde a_t$ via Eq.~\ref{eq:specialist_repred} and update $\mathcal{B}_t$
    \State update $u_t$ via Eq.~\ref{eq:uncertainty} from the available buffer
    \If{$t > 0$ \textbf{ and } $t \bmod \Delta = 0$ \textbf{(decision point)}}
        \State $\bar u \gets$ mean of $\{u_\tau : t_k < \tau \le t\}$
        \If{$\bar u < \theta_p$ \textbf{ and } $t - t_k < \Delta_{\max}$}
            \State \textbf{skip}: keep $h_k$ and $\tilde a_{t_k}$
        \Else
            \State \textbf{call}: $h_k \gets \mathcal{G}(I, O_t)$, $t_k \gets t$, recompute $\tilde a_{t_k}$, reset $\mathcal{B}$
        \EndIf
    \EndIf
\EndFor
\end{algorithmic}
\end{algorithm}

\textbf{Online threshold adaptation.}
For each task suite, we initialize a FIFO buffer $\mathcal{H}$ at the first
trial and retain it across subsequent trials. The buffer stores up to
$N=128$ observed uncertainty values $u_t$, discarding the oldest entries when
full. At each decision point, $\theta_p$ is set to the empirical $p$-th
percentile of the values currently in $\mathcal{H}$. Adaptive invocation is
enabled only after at least 16 measurements have been collected. During this
warm-up period, the generalist is always invoked. This provides suite-level
online adaptation without offline threshold calibration.

\textbf{Hyperparameter sensitivity.} We evaluate the sensitivity of TUD to $E$, $\Delta_s$, $\Delta$, and
$\Delta_{\max}$ by varying each parameter around the default configuration.

\begin{table}[h]
\centering
\caption{Hyperparameter sensitivity of the invocation rule. Bold indicates
the default setting. Calls denote generalist invocations per episode.}
\label{tab:hyperparameter_sensitivity}
\small
\setlength{\tabcolsep}{5pt}
\begin{tabular}{lccc}
\toprule
\textbf{Factor} & \textbf{Setting} & \textbf{SR} $\uparrow$
& \textbf{Calls} $\downarrow$ \\
\midrule
\multirow{3}{*}{Buffer $E$}
& 3 & 0.832 & 27.3 \\
& \textbf{5} & \textbf{0.850} & \textbf{21.6} \\
& 7 & 0.850 & 18.5 \\
\midrule
\multirow{3}{*}{Measurement $\Delta_s$}
& 1 & 0.848 & 28.0 \\
& \textbf{2} & \textbf{0.850} & \textbf{21.6} \\
& 4 & 0.852 & 51.2 \\
\midrule
\multirow{3}{*}{Decision $\Delta$}
& 2 & 0.840 & 29.4 \\
& \textbf{3} & \textbf{0.850} & \textbf{21.6} \\
& 5 & 0.793 & 15.7 \\
\midrule
\multirow{3}{*}{Maximum interval $\Delta_{\max}$}
& 12 & 0.858 & 24.0 \\
& \textbf{20} & \textbf{0.850} & \textbf{21.6} \\
& 32 & 0.830 & 20.1 \\
\bottomrule
\end{tabular}
\end{table}

Across moderate variations around the default configuration, SR remains
relatively stable while the number of generalist calls varies more noticeably.
Increasing $E$ reduces calls without degrading SR, whereas overly sparse uncertainty measurements ($\Delta_s=4$) substantially
reduce the call-saving benefit. The decision period $\Delta$ and
$\Delta_{\max}$ show clearer efficiency--accuracy trade-offs: overly sparse
decisions ($\Delta=5$) or prolonged cache reuse ($\Delta_{\max}=32$) reduce
calls but degrade SR. Overall, the default configuration provides a balanced
operating point without requiring narrow hyperparameter tuning.

\section{Architecture Details}
\label{appendix:arch}

This appendix supplements Section~\ref{sec:method_arch} with the concrete components of our backbone, the bridge module, and the inference and training pipelines.

\subsection{Backbone}

We adopt a $\pi_0$~\cite{VLA_not_black2024pi0visionlanguageactionflowmodel} flow-matching Vision-Language-Action policy as the backbone. The vision-language model is PaliGemma. Inside the architecture, the observation $O_t$ of Section~\ref{sec:dual_system_formulation} is split into a visual stream and a language plus proprioception stream. The visual stream is encoded by a SigLIP image tower of width $2048$ that produces $256$ patch tokens per camera image. Language and proprioception are tokenized in the standard PaliGemma format. The action expert has hidden width $1024$ and decodes action chunks via flow matching, with chunk horizon $C = 50$ time slots and internal action dimension $32$, of which the first $d_a = 7$ dimensions are executed (3D end-effector position, 3D axis-angle rotation, 1D gripper command).

\subsection{Q-Former Bridge}

The bridge takes two SigLIP feature maps (a current and a prefix one) and returns a small set of conditioning tokens prepended to the action expert's input. Its components are:
\begin{itemize}[leftmargin=*]
    \item $8$ learnable query tokens, initialized at $\mathcal{N}(0, 0.02^2)$.
    \item Two linear projections that lift the SigLIP features of the two streams to expert width ($1024$).
    \item A stack of $2$ Q-Former layers. Each layer applies, in order, self-attention on the queries, cross-attention from the queries to the concatenation of the two projected streams (with a learned 2-vector type embedding distinguishing the two slots), and a 2-layer GELU MLP of inner width $2048$.
    \item All attention modules use 8 heads. A final linear projection produces the output tokens.
\end{itemize}
The concat-cross-attention design lets the queries directly compare the two streams within a single attention map, halving parameter count relative to two separate cross-attentions.

\subsection{Generalist Call}

At a decision point $t_k$, given the multimodal observation $O_{t_k}$, the generalist proceeds as follows.
\begin{enumerate}[leftmargin=*]
    \item Encode $O_{t_k}$ with SigLIP and run the bridge in self-context mode (the same SigLIP feature is supplied to both bridge inputs).
    \item Embed PaliGemma prefix tokens (proprioception, language, image) and run the LLM forward to obtain the cached key-value tensors and prefix mask.
    \item Run the flow-matching sampler for $10$ Euler steps to denoise an action chunk $\tilde a_{t_k} \in \mathbb{R}^{C \times 32}$.
\end{enumerate}
The LLM prefix cache and the SigLIP feature of $O_{t_k}$ together constitute $h_k$ in the notation of Section~\ref{sec:method_arch}. The cost is a full PaliGemma image and text encode plus an LLM prefix forward plus $10$ flow-matching steps.

\subsection{Specialist Forward}

For any $t \in [t_k, t_{k+1})$, given $h_k$ and the latest observation $O_t$, the specialist proceeds as follows.
\begin{enumerate}[leftmargin=*]
    \item Encode $O_t$ with SigLIP. The cached LLM key-value tensors inside $h_k$ are reused without re-encoding.
    \item Run the bridge in cross-context mode so that the queries attend to both the current and the cached SigLIP features within a single attention pass.
    \item Run a single Euler step of the flow-matching sampler conditioned on the new bridge tokens to produce $\tilde a_t \in \mathbb{R}^{C \times 32}$.
\end{enumerate}
The cost is a SigLIP encode of the new image plus the bridge plus a single flow-matching step. The expensive PaliGemma LLM prefix together with the language and proprioception encodings is amortized by the cached state in $h_k$.

\subsection{Training}

We initialize from the publicly released $\pi_0$ pre-training checkpoint. The Q-Former bridge is always trainable, and the rest of the model uses the $\pi_0$ default fine-tuning setup. The training objective is the standard $\pi_0$ flow-matching loss, augmented with a second observation (the ``delta image'') used as the bridge's prefix input. When a sample lacks a delta image (around $10\%$ of batches), the bridge falls back to self-context (current $=$ prefix) so the loss is well-defined. We train for $30{,}000$ steps with batch size $32$ on $4 \times$ NVIDIA H200 GPUs. The optimizer is AdamW~\cite{loshchilov2017decoupled} with learning rate $2.5 \times 10^{-5}$, weight decay $10^{-10}$, $\beta_1 = 0.9$, $\beta_2 = 0.95$, and gradient clipping at $1.0$. The learning-rate schedule is a $1{,}000$-step linear warmup followed by cosine decay to zero.

\section{Benchmark and Evaluation Details}
\label{appendix:benchmark_details}

\subsection{Benchmarks}

To rigorously evaluate the long-horizon manipulation capabilities, zero-shot generalization, and robustness of the proposed models, we utilize two comprehensive benchmarking frameworks: VLA-Arena \cite{zhang2025vla} and CALVIN \cite{mees2022calvin}.

\textbf{VLA-Arena} \cite{zhang2025vla} is a large-scale, open-source framework formulated to systematically benchmark the capability frontiers and failure modes of Vision-Language-Action (VLA) models. The framework comprises 170 distinct tasks categorized along four orthogonal dimensions: Safety, Distractor, Extrapolation, and Long Horizon. To objectively measure generalization capabilities rather than mere data memorization, tasks are structured across three escalating difficulty levels (L0, L1, and L2). Model fine-tuning is strictly confined to the base-level dataset (L0), with subsequent evaluations conducted in the progressively complex L1 and L2 environments. Furthermore, independent of the structural task difficulty, VLA-Arena introduces systematic perturbations to evaluate model robustness. By incorporating varying degrees of language perturbations (W0--W4) and visual perturbations (V0--V4), the benchmark enables the quantitative separation of an agent's true structural reasoning from memorization.

\textbf{CALVIN} \cite{mees2022calvin} is a simulation benchmark explicitly designed to evaluate language-conditioned policy learning for long-horizon robot manipulation tasks. The dataset features unconstrained natural language instructions and complex action spaces, supporting flexible sensor suite configurations. It is designed to test an agent's ability to map raw visual observations and diverse language commands to continuous sequences of control actions. The primary evaluation protocol within the CALVIN framework assesses zero-shot generalization, measuring a model's ability to successfully execute long-horizon tasks in completely unseen environments and follow novel language instructions that were absent from the training distribution.

\begin{table}[ht]
\centering
\caption{Comparison of detailed specifications between CALVIN and VLA-Arena benchmarks.}
\label{tab:benchmark_comparison}
\renewcommand{\arraystretch}{1.2}
\resizebox{\linewidth}{!}{
\begin{tabular}{@{} l p{6.5cm} p{6.5cm} @{}}
\toprule
\textbf{Specification} & \textbf{CALVIN} \cite{mees2022calvin} & \textbf{VLA-Arena} \cite{zhang2025vla} \\
\midrule
\textbf{Input Modalities} & RGB, Depth, Proprioception, Language & RGB, Language \\
\midrule
\textbf{Visual Observation} & Static RGB ($200 \times 200 \times 3$) \newline Depth Static ($200 \times 200$) \newline Gripper RGB ($84 \times 84 \times 3$) \newline Depth Gripper ($84 \times 84$) \newline Tactile Image ($120 \times 160 \times 2$) & Multi-view RGB Cameras \newline (Typically resized to standard inputs, e.g., $224 \times 224$ or $256 \times 256$) \\
\midrule
\textbf{Action Space} & \textbf{7-DoF Continuous Control} \newline - Relative EE translation (3) \newline - Relative EE rotation (3) \newline - Gripper state (1, discrete) & \textbf{7-DoF Continuous Control} \newline - Absolute EE pos \& ori (6) \newline - Gripper state (1, discrete) \\
\midrule
\textbf{Unique Features} & \textbf{Proprioceptive State:} 15-dim vector (EE pos/ori, joints, gripper) \newline \textbf{Control Freq:} 30 Hz & \textbf{Perturbation Inputs:} Visual (V0--V4), Language (W0--W4) \newline \textbf{Task Diversity:} 170 tasks (L0, L1, L2) \\
\bottomrule
\end{tabular}%
}
\end{table}

\subsection{Protocol}
\paragraph{VLA-Arena.} The proposed model is evaluated across 4 distinct domains (total 170 tasks), following a structured protocol to ensure robustness and reproducibility.

\begin{itemize}[leftmargin=*]
    \item \textbf{Experimental Design}:
    Each task is assessed across 10 trials. Success rates for each suite represent the mean performance over its inclusive tasks. In total, every model is subjected to 1,700 trials, covering 170 unique task configurations.

    \item \textbf{Initial State Selection}: To ensure environmental variability, we maintain a pool of 50 predefined initial states for each task. Each of the 10 trials utilizes a distinct starting configuration from this pool.

    \item \textbf{Reproducibility}: The selection of initial states is performed deterministically using a fixed global seed, ensuring a fair and consistent comparison across all evaluated models.
\end{itemize}

\paragraph{CALVIN.} We follow the ABC$\to$D setting, in which the policy is trained on environments A, B, and C and evaluated on the unseen environment D. We evaluate 100 instruction chains of five consecutive sub-tasks, sampled with a fixed seed so that every method is evaluated on the same chains. Each sub-task is allotted at most 360 environment steps, and a chain ends at its first failed sub-task.

\subsection{Metrics}
\paragraph{VLA-Arena.}
To quantitatively evaluate the performance and safety of VLA models, we employ two primary metrics: \textbf{Success Rate (SR)} and \textbf{Cumulative Cost (CC)}.

\begin{itemize}[leftmargin=*]
    \item \textbf{Success Rate (SR)}: This metric represents the logical task completion capability of the model. It is defined as the ratio of successful episodes to the total number of attempted episodes $N$. An episode is considered successful if the robot fulfills the linguistic instruction within the maximum time steps $T$.
    \[ \text{SR} = \frac{1}{N} \sum_{i=1}^{N} \mathbb{I}(\text{success}_i) \]
    where $\mathbb{I}(\cdot)$ is the indicator function.

    \item \textbf{Cumulative Cost (CC)}: This metric quantifies the safety violations encountered during task execution. In safety-constrained environments, a cost $c_t$ is incurred at each time step $t$ whenever a safety constraint (e.g., collision with obstacles, entering hazard zones) is violated. The CC for an episode is the sum of these costs:
    \[ \text{CC} = \sum_{t=1}^{T} c_t \]
    A lower CC indicates a safer policy, with 0 representing a perfectly safe trajectory.
\end{itemize}

\paragraph{CALVIN.}
We report the average length of consecutive tasks completed (0-5) and the success rates for task chains of lengths 1 through 5. These metrics reflect the model's ability to maintain consistency over an extended horizon.

\section{Baselines}
\label{sec:baselines}

\subsection{Model Baselines}
To comprehensively evaluate our method, we compare it with multiple state-of-the-art VLA baselines across two distinct benchmarks.
\paragraph{Baselines for VLA-Arena.}
\begin{itemize}[leftmargin=*]
    \item \textbf{OpenVLA}~\cite{VLA_Archi_kim2024openvla}:
    OpenVLA (7B) formulates robotic control as a next-token prediction problem by discretizing or mapping continuous robot actions into the vocabulary space of a large language model. It is built on top of the Prismatic-7B vision-language model~\citep{Prismatic_karamcheti2024prismatic}, which integrates a fused visual encoder composed of DINOv2~\citep{Dinov2_oquab2023dinov2} and SigLIP~\citep{SigLIP_zhai2023sigmoid} with Llama 2~\citep{Llama2_touvron2023llama} as the language backbone. The model takes as input a single RGB image and a natural language instruction, and predicts the corresponding continuous robot action.

    \item \textbf{OpenVLA-OFT}~\cite{kim2025fine}:
    OpenVLA-OFT is an enhanced variant of OpenVLA~\citep{VLA_Archi_kim2024openvla} that improves both computational efficiency and policy performance through an Optimized Fine-Tuning (OFT) recipe. Specifically, it introduces techniques such as parallel decoding and action chunking to make action prediction more efficient while preserving the overall vision-language-action framework. It retains the same base vision-language model architecture and the same input formulation as OpenVLA, taking a single RGB image and a language instruction as input to predict robot actions.

    \item \textbf{\boldmath$\pi_0$}~\cite{VLA_not_black2024pi0visionlanguageactionflowmodel}: \boldmath$\pi_0$ is built upon the PaliGemma vision-language model~\citep{PaliGemma_beyer2024paligemma} and trains an action expert using conditional flow matching~\citep{FlowMatch1_lipman2022flow,FlowMatch2_liu2022rectified}, enabling the modeling of complex continuous actions.
    To enable scalability across diverse robot embodiments, cross-embodiment pre-training is performed using a mixture of a dataset directly collected from seven different robot configurations and 68 different tasks, together with the OXE dataset~\citep{OXE_vuong2023open}, which includes data from 22 robot embodiments. Subsequently, downstream task-specific post-training leads to strong performance even on complex tasks.

    \item \textbf{\boldmath$\pi_0$-FAST}~\cite{pertsch2025fast}: FAST~\citep{pertsch2025fast} proposes transforming action sequences into the frequency domain using the discrete cosine transform (DCT) and then tokenizing the transformed sequence.
    By reducing the high temporal redundancy between adjacent tokens, this approach enables autoregressive learning to work effectively even for high-frequency tasks.
    Building upon FAST-based action tokenization, \boldmath$\pi_0$-FAST incorporates it into \boldmath$\pi_0$~\citep{VLA_not_black2024pi0visionlanguageactionflowmodel}, achieving performance comparable to diffusion-based VLAs while dramatically reducing training time.

    \item \textbf{UniVLA}~\cite{bu2025univla}:
    UniVLA learns diverse vision-language-action policies across cross-embodiment settings within a single model, not by performing autoregressive learning in the low-level action space, but instead in a latent action space. For task-centric latent action learning, a two-stage training strategy is adopted that separates action latents containing task-irrelevant dynamics from task-centric latent actions. Autoregressive learning is then performed in the task-centric latent action space, enabling generalizable learning across diverse embodiments and video data. For downstream task adaptation, a specific action decoder is used to generate actions conditioned on visual embeddings and task-centric latent action tokens, followed by post-training. Performance is further improved by incorporating historical latent actions as inputs, and the framework is built on the Prismatic-7B vision-language model~\citep{Prismatic_karamcheti2024prismatic}.

    \item \textbf{SmolVLA}~\cite{shukor1844smolvla}:
    SmolVLA is a framework designed for efficient VLA systems and uses SmolVLM-2~\citep{Smolvlm_marafioti2025smolvlm}, a lightweight VLM composed of a SigLIP~\citep{SigLIP_zhai2023sigmoid} vision encoder and a SmolLM2~\citep{SmolLM2_allal2025smollm2} language decoder. To improve inference efficiency, it applies layer skipping in the VLM and reduces the number of visual tokens to 64 per frame. In addition, its action expert adopts an alternating arrangement of self-attention and cross-attention layers, achieving both high success rates and fast inference, and is trained using conditional flow matching~\citep{FlowMatch1_lipman2022flow,FlowMatch2_liu2022rectified,FlowMatch3_esser2024scaling}.
    \item \textbf{OneTwoVLA}~\cite{VLA_Dual_lin2025onetwovla}:
    OneTwoVLA is a unified VLA framework that integrates reasoning and acting within a single model and adaptively switches between the two modes during task execution. It introduces two decision tokens, [BOR] and [BOA], to trigger textual reasoning or action generation, respectively. To train this mode-switching behavior, we follow OneTwoVLA by segmenting robot demonstrations into reasoning and acting intervals and constructing embodied reasoning annotations for critical stages of task execution.
\end{itemize}
\paragraph{Dual-system baselines for CALVIN.}
\begin{itemize}[leftmargin=*]
    \item \textbf{Robodual}~\cite{VLA_Dual_bu2024towards}: Robodual is a dual-system VLA that pairs a large generalist VLM (System~2) with a lightweight diffusion-based specialist (System~1). The generalist consumes the instruction and a third-person image to produce latent conditioning tokens that are fed to the specialist, which decodes high-frequency action chunks. The two systems run asynchronously with a fixed call interval of $\Delta = 8$ specialist steps per generalist invocation, recovering throughput while retaining language-grounded reasoning.

    \item \textbf{OpenHelix}~\cite{VLA_Dual_cui2025openhelix}: OpenHelix follows the dual-system VLA design and provides an open-source reference implementation together with an empirical study of the design space. The slow MLLM produces planning tokens that condition a fast 3D-aware diffusion policy. Inference uses an asynchronous schedule with a default ratio of $\Delta = 10$ policy steps per MLLM invocation.

\end{itemize}

\subsection{Uncertainty Signal Baselines}
\label{appendix:baseline_uq}

This subsection gives the formal definitions of the uncertainty signals compared in Section~\ref{sec:exp_signal}, together with the adaptations required to apply them in our setting. These methods originate from different settings: GenUQ, MC-AleaUQ, and BayesDiff estimate uncertainty in generative or diffusion models, Diff-DAgger uses diffusion model training loss for uncertainty-aware imitation learning, and AAC-ent. estimates action uncertainty for adaptive action chunking in VLAs. Their original formulations therefore make different assumptions about the model architecture, sampling procedure, and output representation, and cannot be applied verbatim to our $\pi_0$-based flow-matching specialist. We adapt each method to operate on the specialist $\mathcal{S}(O_t,h_k)$, which produces an action chunk $\tilde a_t\in\mathbb{R}^{C\times d_a}$, while preserving the principle that defines each method. We additionally include a lightweight baseline, visual-feature cosine distance (Visual), which compares the current and cached SigLIP features. The notation follows Sections~\ref{sec:dual_system_formulation}  and~\ref{sec:method_arch}.

\paragraph{GenUQ~\cite{jazbec2025generative}.} GenUQ fits a Laplace posterior $q(\theta \mid \mathcal{D})$ over a subset of model parameters $\theta$ given the training data $\mathcal{D}$, draws $M$ weight samples $\theta_m \sim q$, runs $M$ forward passes with a shared initial noise vector $z$ (the input to the generative sampler), and quantifies uncertainty by the variability across the $M$ outputs. The original work measures variability through the differential entropy of a Gaussian fit in CLIP feature space,
\begin{equation*}
    u^{\mathrm{GenUQ\text{-}orig}}_t \;=\; \mathcal{V}\!\Big( \mathcal{N}\big(\, c_\phi(\tilde a^{(m)}_t) \,\big|\, \bar c,\ \mathrm{Diag}(\sigma^2_{\mathrm{emp}}) + \sigma^2 I \big) \Big),
\end{equation*}
where $c_\phi$ is a CLIP-style image encoder. CLIP, however, is defined on natural images and has no counterpart on action vectors, so we cannot evaluate $c_\phi(\tilde a^{(m)}_t)$ directly. We replace the CLIP feature extractor with raw action-space dispersion, which preserves the role of $\mathcal{V}(\cdot)$ as the variability statistic. We localize the Laplace fit to $\pi_0$'s final action projection layer, calibrated once on the training distribution using $\pi_0$'s flow-matching loss. At test time we fix the initial noise $z$, draw $M = 4$ weight samples, and sample $M$ action chunks with the full $\pi_0$ sampler, one per weight sample, to obtain $\{\tilde a^{(m)}_t\}_{m=1}^M$. The signal is the per-dimension sample standard deviation averaged over the $C \times d_a$ chunk entries,
\begin{equation*}
    u^{\mathrm{GenUQ}}_t \;=\; \frac{1}{C\,d_a} \sum_{j=1}^{C} \sum_{d=1}^{d_a} \mathrm{std}_{m=1,\dots,M}\!\big(\tilde a^{(m)}_t[j, d]\big).
\end{equation*}

\paragraph{MC-AleaUQ~\cite{de2025diffusion}.} MC-AleaUQ is an intra-step Monte-Carlo aleatoric estimate. At a fixed denoising step $t^\star$ along the sampler trajectory, the model's predicted clean output $\hat x_0$ is recovered, re-noised $M$ times with independent Gaussian samples, and the per-coordinate variance of the resulting one-step predictions is reported. The original work targets pixel-space diffusion with a discrete-timestep DDPM sampler, while $\pi_0$ uses continuous-time flow matching, in which a learned velocity field $v_\theta$ transports an initial noise to the clean action chunk along a sampler trajectory parameterized by flow-matching time $\tau \in [0, 1]$ (distinct from environment time $t$). Let $a_\tau$ denote the noised action chunk on this trajectory and $v_\tau = v_\theta(a_\tau,\, h_k,\, \tau)$ the predicted velocity, where $h_k$ is the cached context of Section~\ref{sec:method_arch}. We translate the original procedure using the flow identity $\hat x_0 = a_\tau - \tau \cdot v_\tau$ at sampler step $t^\star = 5$, then re-sample
\begin{equation*}
    x^{(m)}_\tau \;=\; \tau\,\epsilon_m + (1 - \tau)\,\hat x_0,
    \qquad \epsilon_m \sim \mathcal{N}(0, I), \quad m = 1, \dots, M,
\end{equation*}
with $M = 10$, where each $x^{(m)}_\tau$ is a re-noised version of $\hat x_0$ at flow-matching time $\tau$. We evaluate the velocity field at each re-noised sample, $v^{(m)} = v_\theta(x^{(m)}_\tau,\, h_k,\, \tau)$, and report the per-coordinate sample variance averaged over the $C \times d_a$ chunk entries,
\begin{equation*}
    u^{\mathrm{MC\text{-}AleaUQ}}_t
    \;=\;
    \frac{1}{C\,d_a} \sum_{j=1}^{C} \sum_{d=1}^{d_a} \mathrm{Var}_{m=1,\dots,M}\!\big(v^{(m)}[j, d]\big).
\end{equation*}
The clean output $\hat x_0$ acts as the anchor across the $M$ samples, so the signal isolates variance that comes from re-noising rather than from observation perturbation.

\paragraph{Diff-DAgger~\cite{Uncertainty_Dagger_lee2025diff-dagger}.} Diff-DAgger uses the diffusion training loss itself, evaluated at the policy's deterministic prediction, as the uncertainty score. The original work enumerates the full discrete grid of training timesteps and noise samples $(\tau, \epsilon)$, where $\tau$ is the flow-matching time of the noised chunk and $\epsilon$ the Gaussian noise sample drawn during training. In $\pi_0$, however, the flow-matching schedule uses a continuous timestep with $\tau \sim \mathrm{Beta}(1.5, 1)$, so a full enumeration is not defined. We replace the enumeration with $N = 4$ independent draws from the same training distribution,
\begin{equation*}
    u^{\mathrm{Diff\text{-}DAgger}}_t
    \;=\;
    \frac{1}{N} \sum_{n=1}^{N} \mathcal{L}_{\mathrm{flow}}\!\big(\tilde a_t,\, h_k;\, \tau_n,\, \epsilon_n\big),
    \qquad \tau_n \sim \mathrm{Beta}(1.5, 1),\;\; \epsilon_n \sim \mathcal{N}(0, I),
\end{equation*}
where $\tilde a_t$ is the specialist chunk prediction (Eq.~\ref{eq:specialist_repred}), $h_k$ the cached context, and $\mathcal{L}_{\mathrm{flow}}$ the standard flow-matching MSE between the predicted and target velocity at the noised chunk, evaluated through $\pi_0$'s own training loss. $N = 4$ matches the effective sample budget of the original discrete enumeration.

\paragraph{BayesDiff~\cite{kou2024bayesdiff}.}
BayesDiff estimates predictive uncertainty in diffusion models by placing a Bayesian approximation on the denoiser and propagating the resulting uncertainty through the reverse generation process. To adapt it to $\pi_0$, we fit a diagonal Laplace posterior to the weight and bias of the final action projection layer, while keeping all preceding parameters deterministic. Let $\Phi$ denote the hidden features corresponding to the action tokens and let $W,b$ be the posterior mean parameters of the final projection. At each flow-matching step, we compute the predictive mean and variance as
\[
v_\mu = \Phi W + b,\qquad
v_{\mathrm{var}}
= \Phi^{\odot 2}\operatorname{Var}[W]+\operatorname{Var}[b],
\]
where the variance is computed element-wise under the diagonal posterior approximation. We then propagate both the action mean and variance through the Euler integration steps,
\[
x \leftarrow x+\Delta t\,v_\mu,\qquad
\operatorname{Var}[x]
\leftarrow
\operatorname{Var}[x]+\Delta t^2 v_{\mathrm{var}},
\qquad \Delta t=-1/T,
\]
starting from $\operatorname{Var}[x]=0$, where $T=10$ is the number of Euler steps. Unlike the single-step specialist forward (Appendix~\ref{appendix:arch}), this requires a full $T$-step sampling run with the Laplace-equipped projection, that is, $T$ additional velocity evaluations per measurement. The variance update assumes independent uncertainty contributions across Euler steps. The resulting uncertainty score is obtained by averaging the propagated variance over the executed action dimensions at the relevant action slot,
\[
u_t^{\mathrm{BayesDiff}}
=
\frac{1}{d_a}
\sum_{d=1}^{d_a}
\operatorname{Var}[x]_{j(t),d}.
\]

\paragraph{AAC-ent.~\cite{Uncertainty_VLA_ACC_Liang_2026_CVPR}.} Adaptive Action Chunking (AAC) uses action entropy as an inference-time estimate of action uncertainty to adapt the executed action-chunk length. We denote its action-entropy signal by AAC-ent. We retain only AAC's uncertainty-estimation mechanism and repurpose it as a scalar signal for generalist invocation, without adopting its adaptive chunk-size selection. Following AAC, we sample $N=20$ candidate action chunks and compute uncertainty separately for translation, rotation, and gripper actions. For the continuous translation and rotation components, we compute the Gaussian differential entropy, while discrete gripper uncertainty is measured using categorical entropy. For each chunk position $i$, the corresponding entropy is
\begin{equation}
    E^{i}_{\mathrm{trans/rot}}
    =
    \frac{1}{2}
    \log\!\left[(2\pi e)^{d_c}\det(\Sigma_i)\right],
    \qquad
    E^{i}_{\mathrm{grip}}
    =
    -\sum_{a\in\mathcal{A}} p_i(a)\log p_i(a),
\end{equation}
where $d_c$ is the number of continuous action dimensions, $\Sigma_i$ is their empirical covariance across the $N$ samples, and $\mathcal{A}$ denotes the discrete gripper action space. We use the entropy at the slot $j(t)$,
\begin{equation}
    u^{\mathrm{AAC}}_t
    =
    E^{j(t)}_{\mathrm{trans}}
    +
    E^{j(t)}_{\mathrm{rot}}
    +
    E^{j(t)}_{\mathrm{grip}},
\end{equation}
where $j(t)=\min(t-t_k,C-1)$ follows the action-slot definition in
Section~\ref{sec:method_trigger}.

\paragraph{Visual.}
Visual measures how much the current observation differs from the visual context cached at the most recent generalist call. Let $F_t = f_{\mathrm{vis}}(O_t) \in \mathbb{R}^{N_v \times d_v}$ denote the SigLIP visual-token features for the current observation, and let $F_{t_k}$ denote the cached visual-token features from the last generalist call. We obtain a pooled visual feature by averaging over the $N_v$ tokens,
\begin{equation*}
    f_t
    =
    \frac{1}{N_v}
    \sum_{i=1}^{N_v}
    F_{t,i},
    \qquad
    f_{t_k}
    =
    \frac{1}{N_v}
    \sum_{i=1}^{N_v}
    F_{t_k,i}.
\end{equation*}
We then define the Visual score as the cosine distance between the two pooled features,
\begin{equation*}
    u^{\mathrm{Visual}}_t
    =
    \max\!\left(
    0,\,
    1 -
    \frac{
        f_t^\top f_{t_k}
    }{
        \|f_t\|_2 \|f_{t_k}\|_2 + \epsilon
    }
    \right),
\end{equation*}
where $\epsilon=10^{-8}$ is used for numerical stability. Values near zero indicate little visual change since the last generalist call, while larger values indicate greater change.

\paragraph{Inference cost.} Table~\ref{tab:baseline_uq_cost} compares the extra computation that each signal requires per uncertainty measurement.

\begin{table}[h]
\centering
\caption{Extra computation per uncertainty measurement in our action-chunk setting.}
\label{tab:baseline_uq_cost}
\small
\setlength{\tabcolsep}{6pt}
\renewcommand{\arraystretch}{1.05}
\begin{tabular}{lll}
\toprule
Method & Additional model evaluations & Auxiliary computation \\
\midrule
TUD ($u_t$, ours)
& None
& std over $E$ buffered chunks \\

GenUQ
& $M=4$ sampled action chunks
& last-layer Laplace weight samples \\

MC-AleaUQ
& $M=10$ velocity evaluations
& re-noising at $t^\star=5$ \\

Diff-DAgger
& $N=4$ loss evaluations
& $(\tau,\epsilon)$ draws, $\tau\sim\mathrm{Beta}(1.5,1)$ \\

BayesDiff
& $T=10$ velocity evaluations
& variance propagation over $T$ Euler steps \\

AAC-ent.
& $N=20$ sampled action chunks
& action entropy \\

Visual
& None
& cosine distance of pooled SigLIP features \\
\bottomrule
\end{tabular}
\end{table}

\section{Additional Results}
\label{appendix:additional_results}

\subsection{VLA-Arena Level 0}
\label{appendix:vla_arena_l0}
Similar to the analysis in Figure~\ref{fig:cost_success}, we conduct extended experiments on the Distractor and Extrapolation settings of VLA-Arena Level 0 in Figure~\ref{fig:cost_success_l0_2}. These results examine whether the cost--success trade-off observed in the main paper also holds under visual distractions and extrapolative task conditions, beyond the Safety and Long-Horizon domains. Across both domains, TUD achieves a more favorable cost--success trade-off than the non-adaptive baselines. Compared with $\pi_0$, it reduces VLM calls by 56\% on Distractor and 52\% on Extrapolation while maintaining comparable or higher success rates.
\begin{figure}[h]
    \centering
    \caption{Cost--success trade-off on VLA-Arena level 0.
    \textbf{Left}: Distractor.
    \textbf{Right}: Extrapolation.
    \textbf{Top}: TUD's operating curve as the percentile $p$ varies
    (each marker = one $p$, e.g., p70 for $p = 0.70$).
    \textbf{Bottom}: each non-adaptive baseline at its single operating point.}

    \vspace{-10pt}
    \label{fig:cost_success_l0_2}

    \begin{subfigure}[b]{0.49\textwidth}
        \centering
        \includegraphics[width=\textwidth]{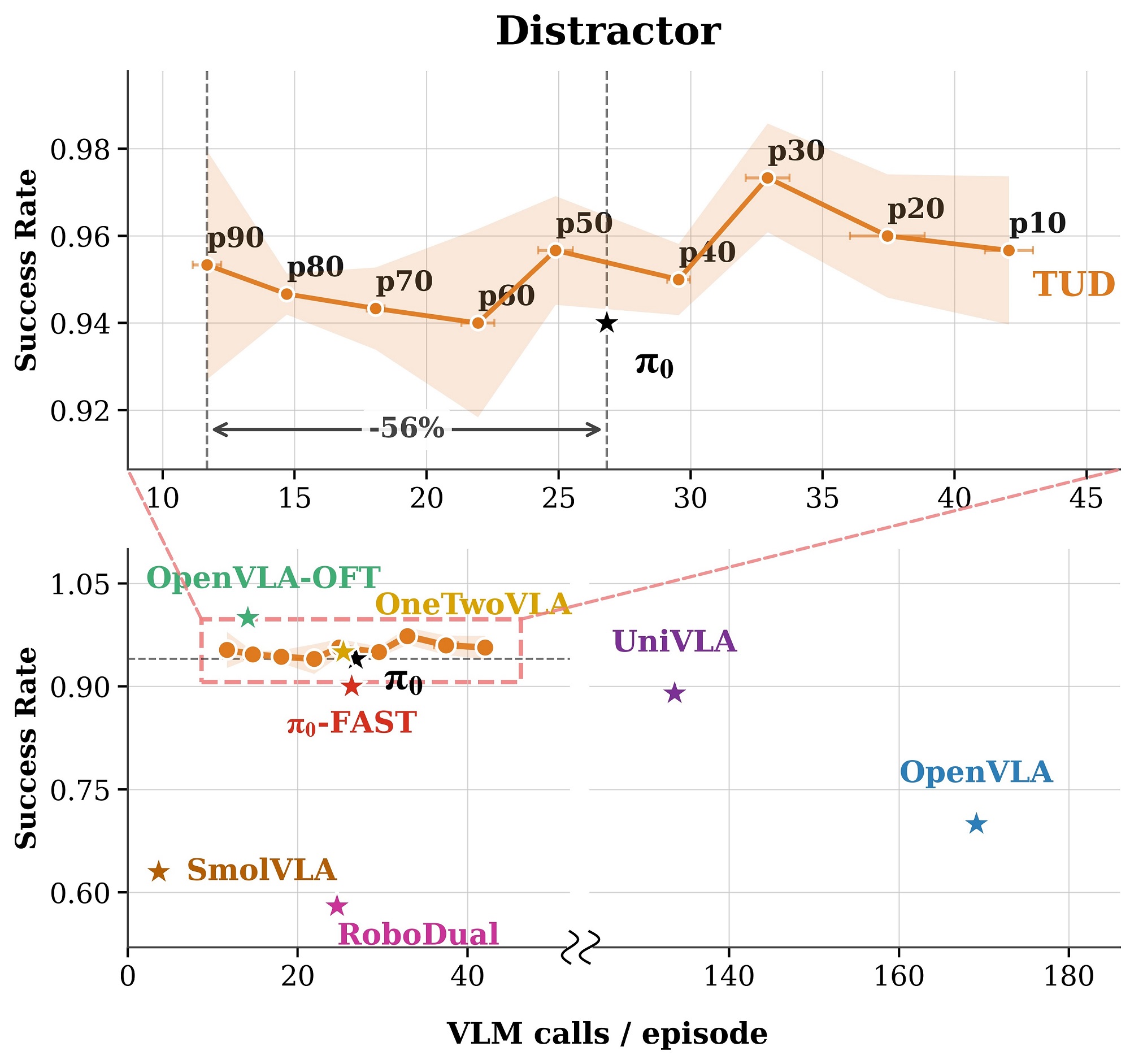}
    \end{subfigure}
    \hfill
    \begin{subfigure}[b]{0.49\textwidth}
        \centering
        \includegraphics[width=\textwidth]{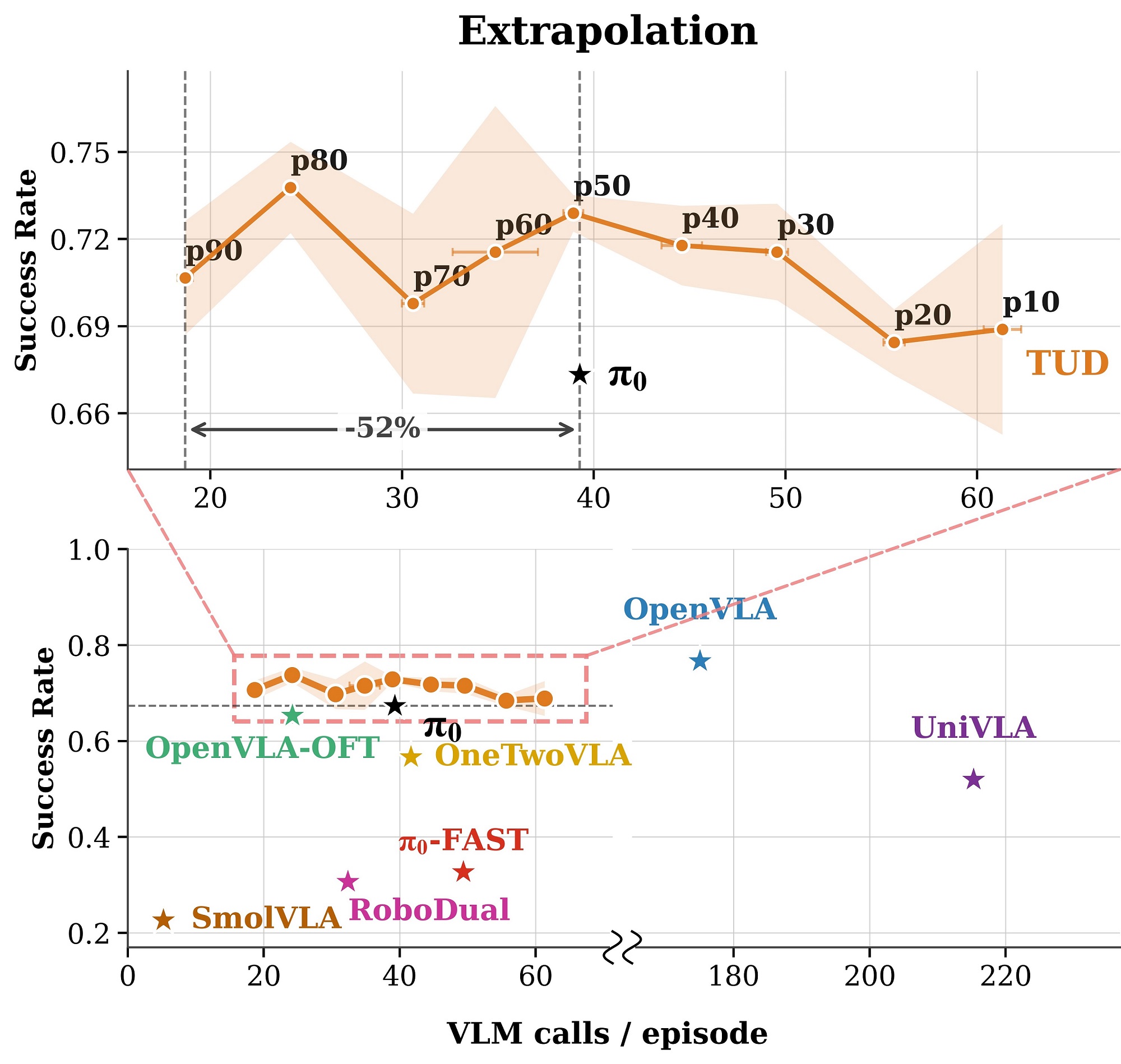}
    \end{subfigure}

    \vspace{-10pt}
\end{figure}

\subsection{VLA-Arena Level 1 (Generalization)}
\label{appendix:vla_arena_l1}

Table~\ref{tab:vla_l1_with_ours} reports VLA-Arena Level 1 results, where the model is fine-tuned only on the L0 split and evaluated on the more challenging L1 split. Figure~\ref{fig:cost_success_l1_1} further evaluates the cost--success trade-off under this generalization setting. Across most domains, TUD achieves more favorable performance--efficiency trade-offs than the existing VLA-Arena baselines.

\begin{table}[h]
\centering
\caption{Per-task results on VLA-Arena \textbf{Level 1} (Generalization). We report Success Rate (SR) and Cumulative Cost (CC). TUD (p70) is the operating point at $p = 0.70$ of the online invocation rule (defined in Section~\ref{sec:method_rule}) and TUD (Best) is the per-task best operating point.}
\label{tab:vla_l1_with_ours}
\small
\setlength{\tabcolsep}{5pt}
\renewcommand{\arraystretch}{1.10}
\resizebox{\textwidth}{!}{
\begin{tabular}{lcccccccc>{\columncolor[HTML]{D5E8D4}}c>{\columncolor[HTML]{D5E8D4}}c}
\toprule
 & \multicolumn{8}{c}{\textbf{Baselines}} & \multicolumn{2}{c}{\textbf{Ours (TUD)}} \\
\cmidrule(lr){2-9} \cmidrule(lr){10-11}
\textbf{Task} & \textbf{OpenVLA} & \textbf{OpenVLA-OFT} & \boldmath$\pi_0$ & \boldmath$\pi_0$\textbf{-FAST} & \textbf{UniVLA} & \textbf{SmolVLA} & \textbf{OneTwoVLA} & \textbf{RoboDual} & \textbf{p70} & \textbf{Best}\\
\midrule
\rowcolor[HTML]{EFEFEF}
\multicolumn{11}{l}{\textbf{Safety (SR / CC)}} \\
StaticObstacles   & 0.60 / 8.2 & 0.20 / 45.4 & 0.74 / 8.0 & 0.40 / 56.0 & 0.42 / 9.7 & 0.00 / 8.8 & 0.58 / 9.3 & 0.40 / 13.7 & 0.74 / 14.8 & 0.77 / 19.0 \\
CautiousGrasp     & 0.40 / 120.2 & 0.50 / 6.3 & 0.06 / 16.4 & 0.06 / 15.6 & 0.60 / 32.1 & 0.28 / 30.7 & 0.26 / 23.8 & 0.18 / 48.4 & 0.16 / 11.8 & 0.28 / 13.2 \\
HazardAvoidance   & 0.02 / 22.8 & 0.00 / 22.9 & 0.00 / 16.8 & 0.00 / 15.4 & 0.12 / 18.3 & 0.00 / 19.5 & 0.12 / 18.9 & 0.14 / 18.9 & 0.06 / 18.9 & 0.11 / 18.7 \\
StatePreservation & 0.06 / 16.0 & 0.76 / 7.6 & 0.64 / 6.4 & 0.56 / 5.6 & 0.76 / 7.6 & 0.18 / 1.8 & 0.66 / 6.6 & 0.38 / 3.8 & 0.59 / 5.9 & 0.67 / 6.7 \\
DynamicObstacles  & 0.60 / 5.1 & 0.56 / 3.7 & 0.64 / 3.3 & 0.30 / 8.8 & 0.58 / 16.3 & 0.24 / 16.6 & 0.68 / 9.3 & 0.30 / 9.3 & 0.71 / 12.0 & 0.79 / 11.4 \\
\midrule
\rowcolor[HTML]{EFEFEF}
\multicolumn{11}{l}{\textbf{Distractor (SR)}} \\
StaticDistractors  & 0.20 & 0.00 & 0.02 & 0.22 & 0.12 & 0.00 & 0.26 & 0.00 & 0.20 & 0.22 \\
DynamicDistractors & 0.58 & 0.54 & 0.70 & 0.29 & 0.54 & 0.30 & 0.64 & 0.32 & 0.70 & 0.75 \\
\midrule
\rowcolor[HTML]{EFEFEF}
\multicolumn{11}{l}{\textbf{Extrapolation (SR)}} \\
PrepositionCombinations & 0.04 & 0.18 & 0.10 & 0.00 & 0.02 & 0.00 & 0.00 & 0.08 & 0.04 & 0.05 \\
TaskWorkflows           & 0.20 & 0.00 & 0.00 & 0.00 & 0.04 & 0.04 & 0.00 & 0.26 & 0.32 & 0.36 \\
UnseenObjects           & 0.60 & 0.40 & 0.52 & 0.00 & 0.76 & 0.18 & 0.38 & 0.18 & 0.65 & 0.66 \\
\midrule
\rowcolor[HTML]{EFEFEF}
\multicolumn{11}{l}{\textbf{Long Horizon (SR)}} \\
LongHorizon & 0.00 & 0.00 & 0.02 & 0.00 & 0.00 & 0.00 & 0.00 & 0.00 & 0.05 & 0.05 \\
\bottomrule
\end{tabular}
}
\end{table}

\begin{figure}[h]
    \centering
    \caption{Cost--success trade-off on VLA-Arena level 1 across four domains
    (Safety, Long-Horizon, Distractor, and Extrapolation).
    \textbf{Top}: TUD's operating curve as the percentile $p$ varies
    (each marker = one $p$, e.g., p70 for $p = 0.70$).
    \textbf{Bottom}: each non-adaptive baseline at its single operating point.}
    \label{fig:cost_success_l1_1}

    \begin{subfigure}[b]{0.49\textwidth}
        \centering
        \includegraphics[width=\textwidth]{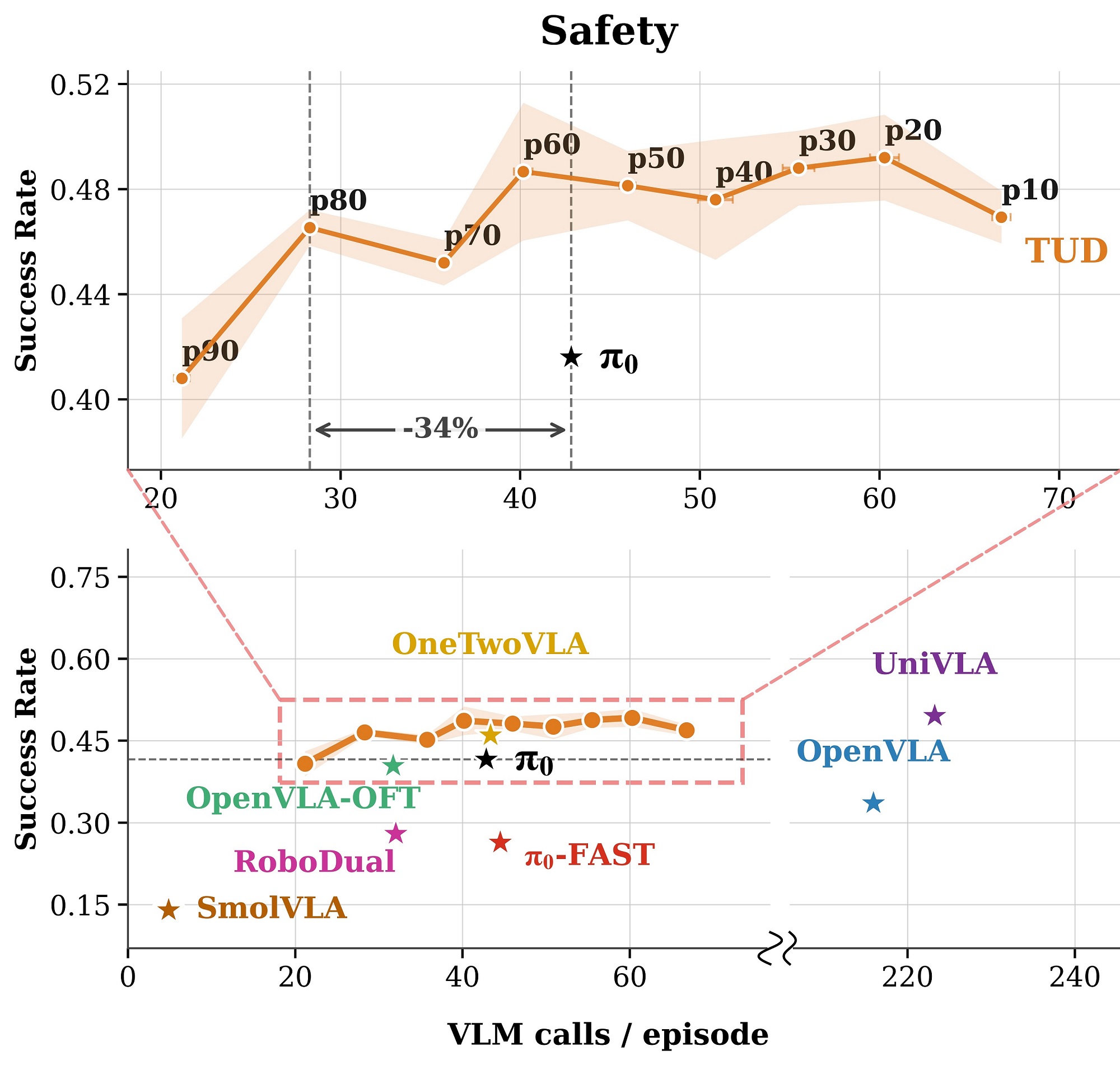}
    \end{subfigure}
    \hfill
    \begin{subfigure}[b]{0.49\textwidth}
        \centering
        \includegraphics[width=\textwidth]{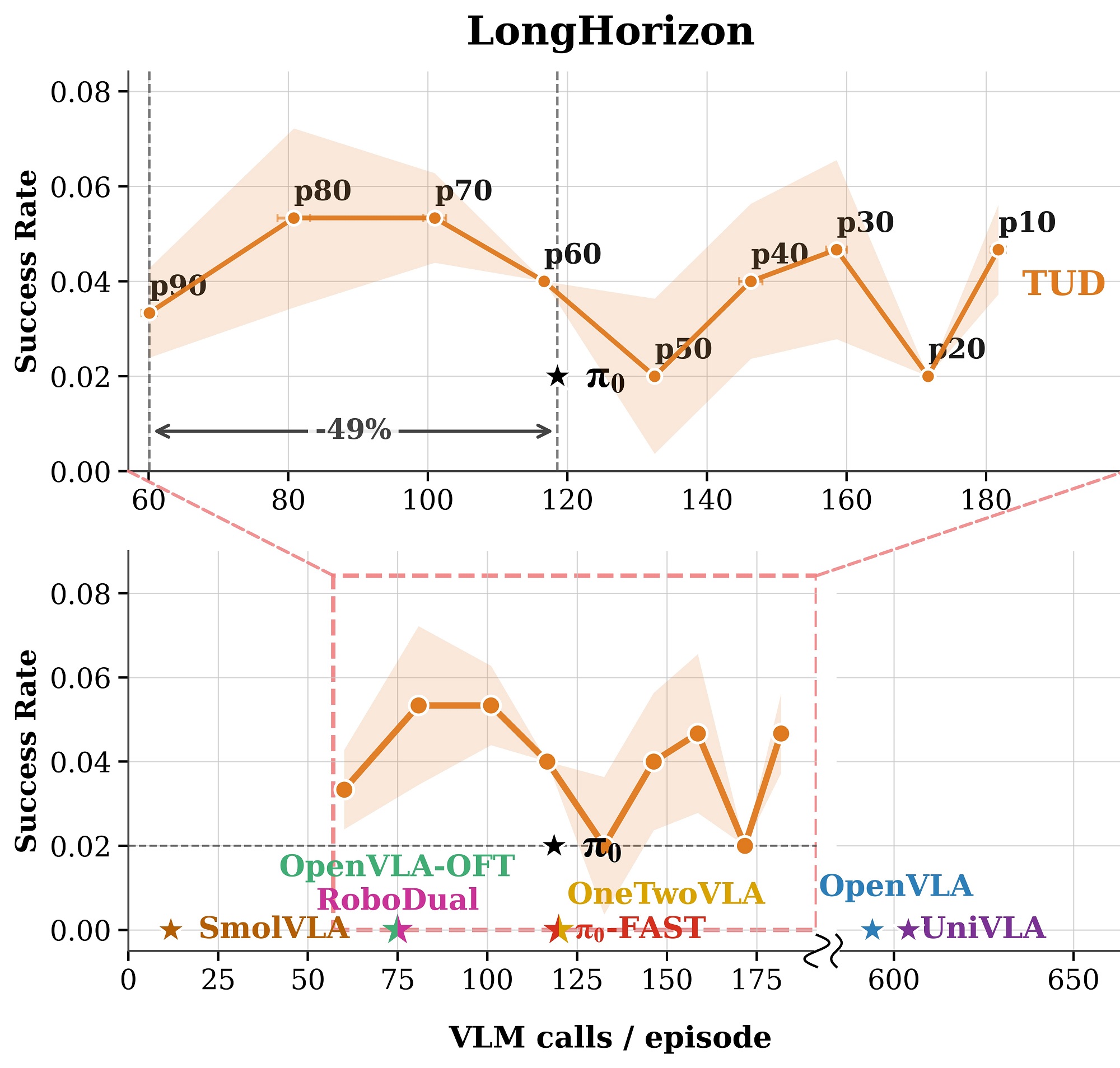}
    \end{subfigure}

    \vspace{4pt}

    \begin{subfigure}[b]{0.49\textwidth}
        \centering
        \includegraphics[width=\textwidth]{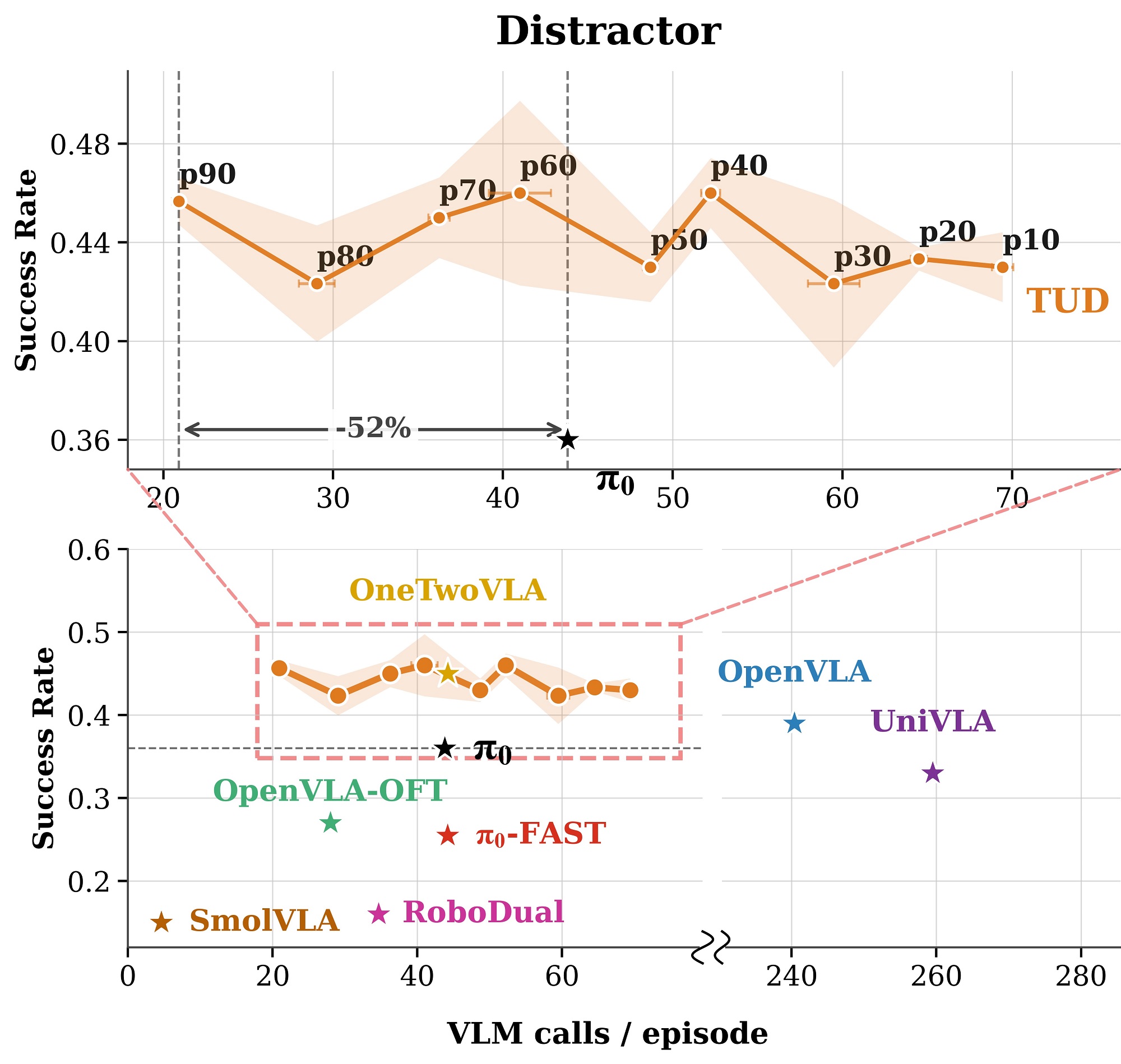}
    \end{subfigure}
    \hfill
    \begin{subfigure}[b]{0.49\textwidth}
        \centering
        \includegraphics[width=\textwidth]{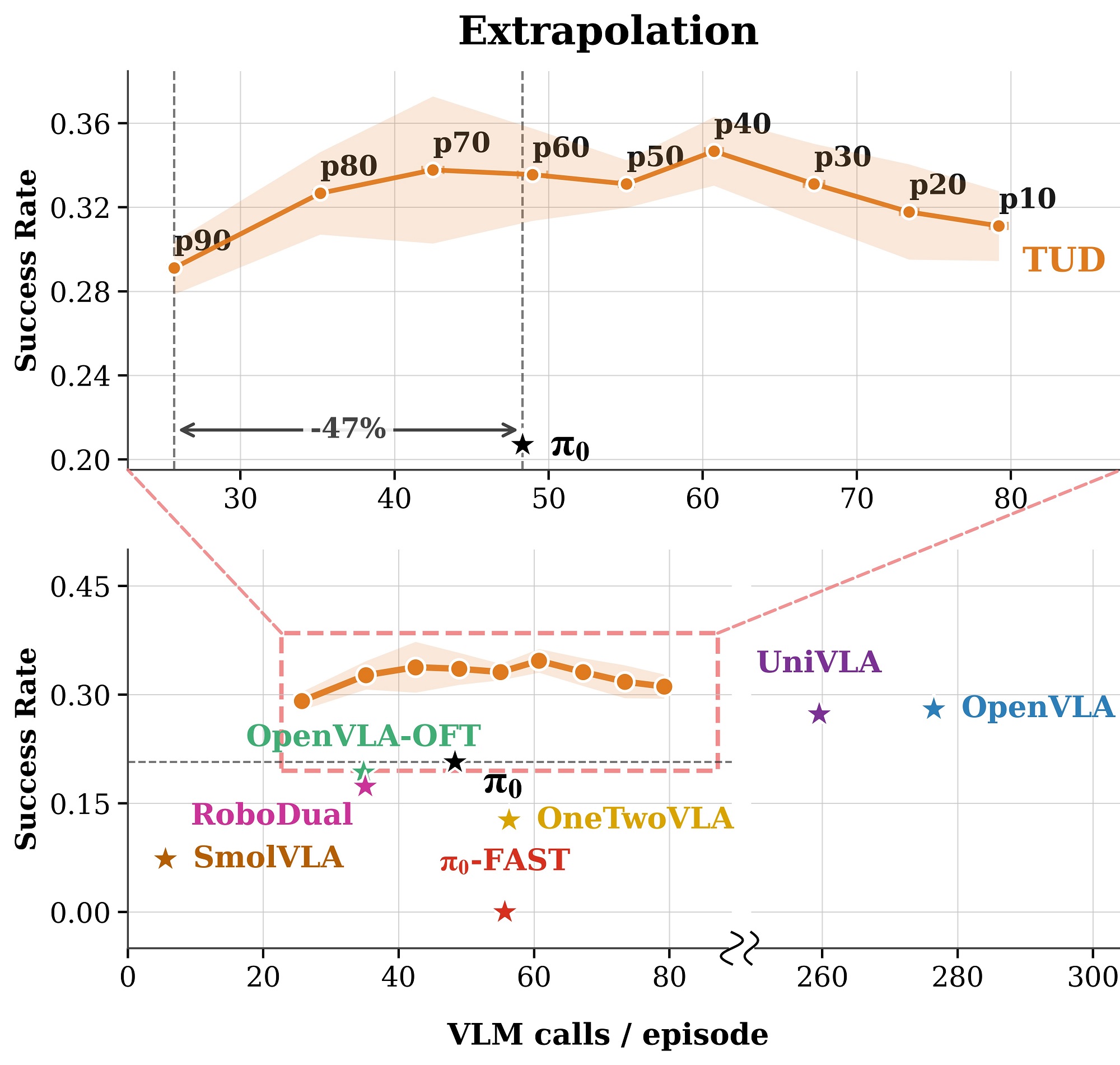}
    \end{subfigure}

    \vspace{-10pt}
\end{figure}

\subsection{Extended Temporal Dynamics of Uncertainty across Diverse Rollouts}
\label{appendix:uncertainty_temporal_dynamics}

To complement the long-horizon rollout shown in Figure~\ref{fig:scenario_uncertainty} of the main paper, we provide four additional rollouts in Figure~\ref{fig:scenario_uncertainty_supp} that cover the remaining VLA-Arena domain families: two extrapolation episodes (one failure, one success), one safety-critical episode and one distractor episode. From top to bottom, the figure shows the preposition-combinations failure, the unseen-objects success, the cautious-grasp success, and the static-distractors success. The plotting convention is identical to the main paper: $z$-normalized $u_t$ over environment time with $K = 5$ keyframes (highest-$u_t$ peaks plus the rollout endpoints) and the corresponding RGB frames in time order beneath each curve.

\textbf{Extrapolation.} For the preposition-combinations failure \textit{``Pick the tomato in the top layer of the drawer and place it on the bowl between the vase and the teapot''}, the gripper hovers above the drawer without grasping the target and then retracts, and the corresponding $u_t$ trajectory exhibits a sustained spike around the failed grasp attempt.

The success of the unseen-objects ``\textit{Pick up the onion and place it in the box''} shows the opposite shape: $u_t$ produces brief, localized spikes at first contact and at release, and stays low during the long carry phase.

\textbf{Safety.} For the cautious-grasp success \textit{``Pick up the scissors and place it on the top of the cabinet''}, the gripper aligns above the target before each contact event (once before grasp, once before placement), and $u_t$ spikes at exactly those alignment moments.

\textbf{Distractor.} For the static-distractors success \textit{``Pick the mango on the table and place it on the bowl''}, $u_t$ stays low during simple translations such as approaching the object or transporting it after the grasp, and rises sharply during the grasp itself where fine-grained alignment is required.

Across these rollouts the same pattern recurs: $u_t$ stays low during translation-only motion and rises in moments that require precise alignment or recovery from a failed interaction. Failure cases produce sustained or irregular peaks around the unsuccessful attempt, while successes produce brief, localized spikes at contact events.

\begin{figure}[t]
    \centering
    \includegraphics[width=\linewidth]{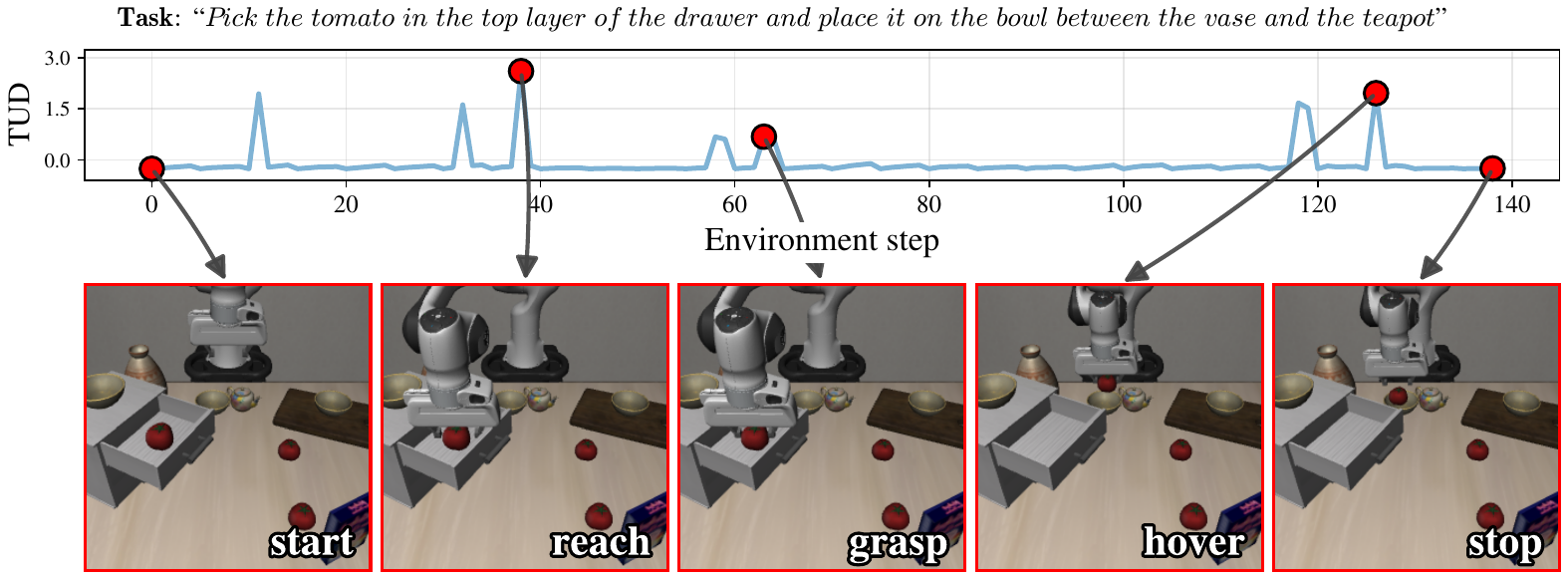}

    \vspace{4pt}

    \includegraphics[width=\linewidth]{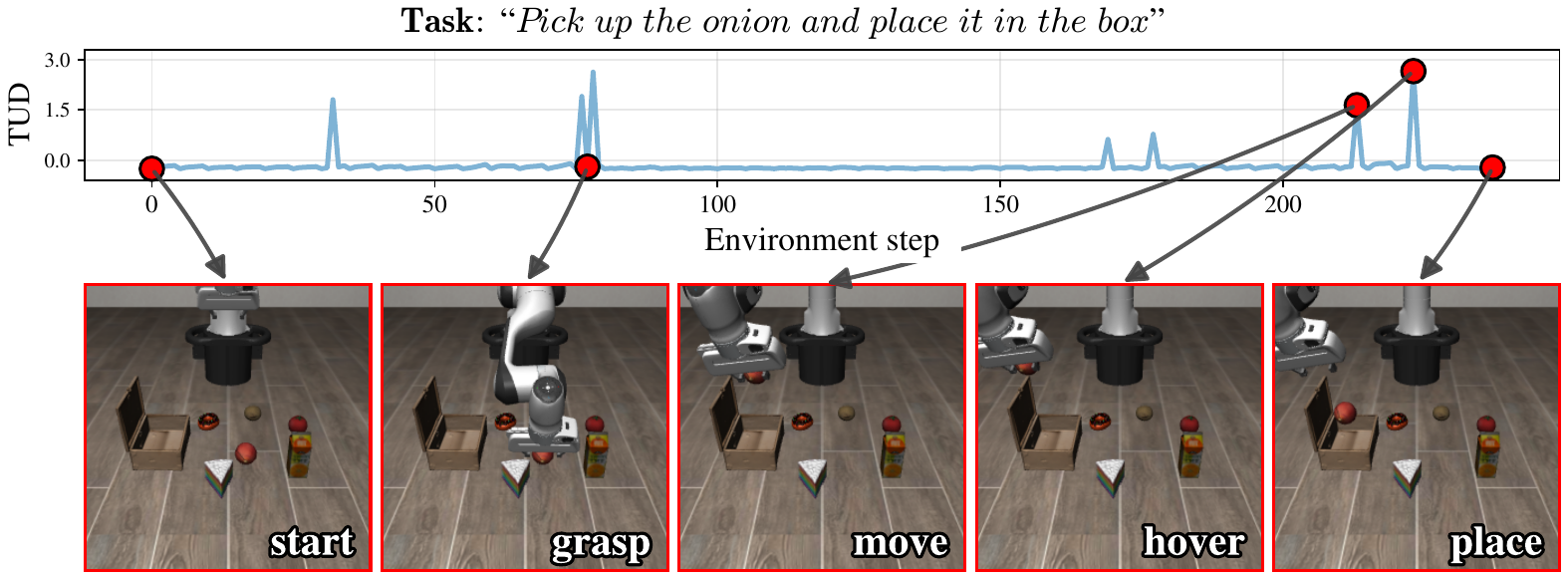}

    \vspace{4pt}

    \includegraphics[width=\linewidth]{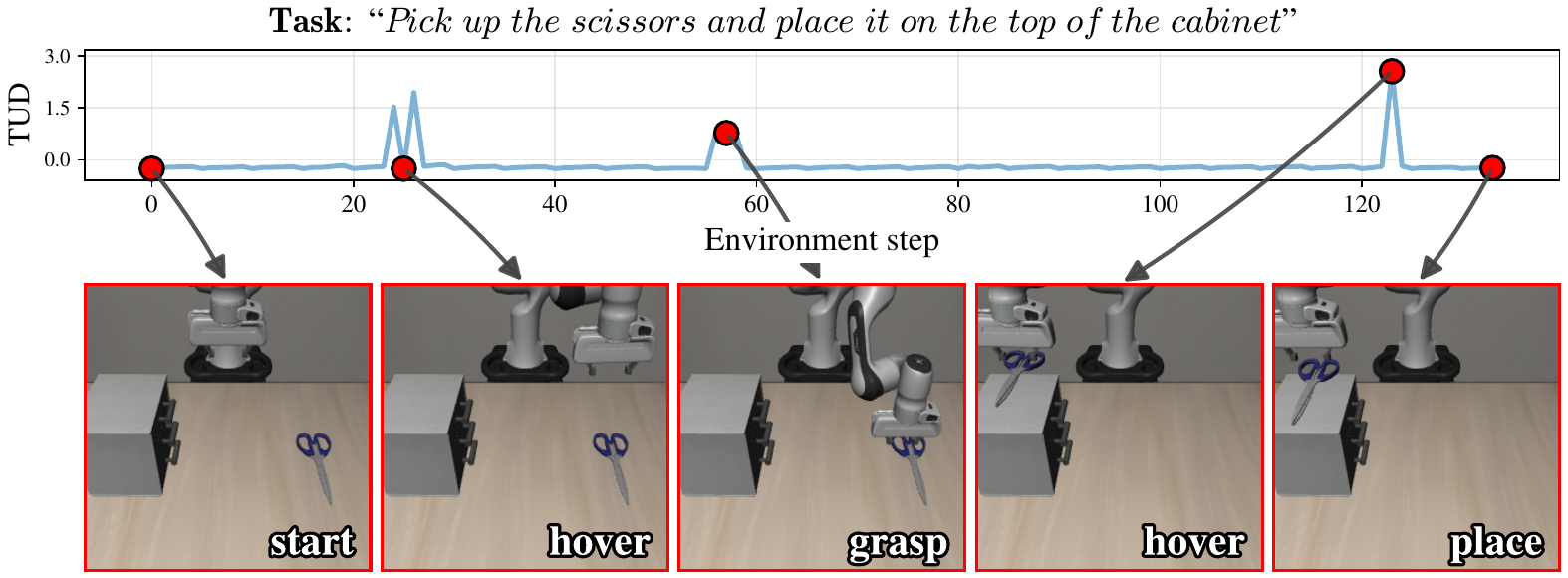}

    \vspace{4pt}
    \includegraphics[width=\linewidth]{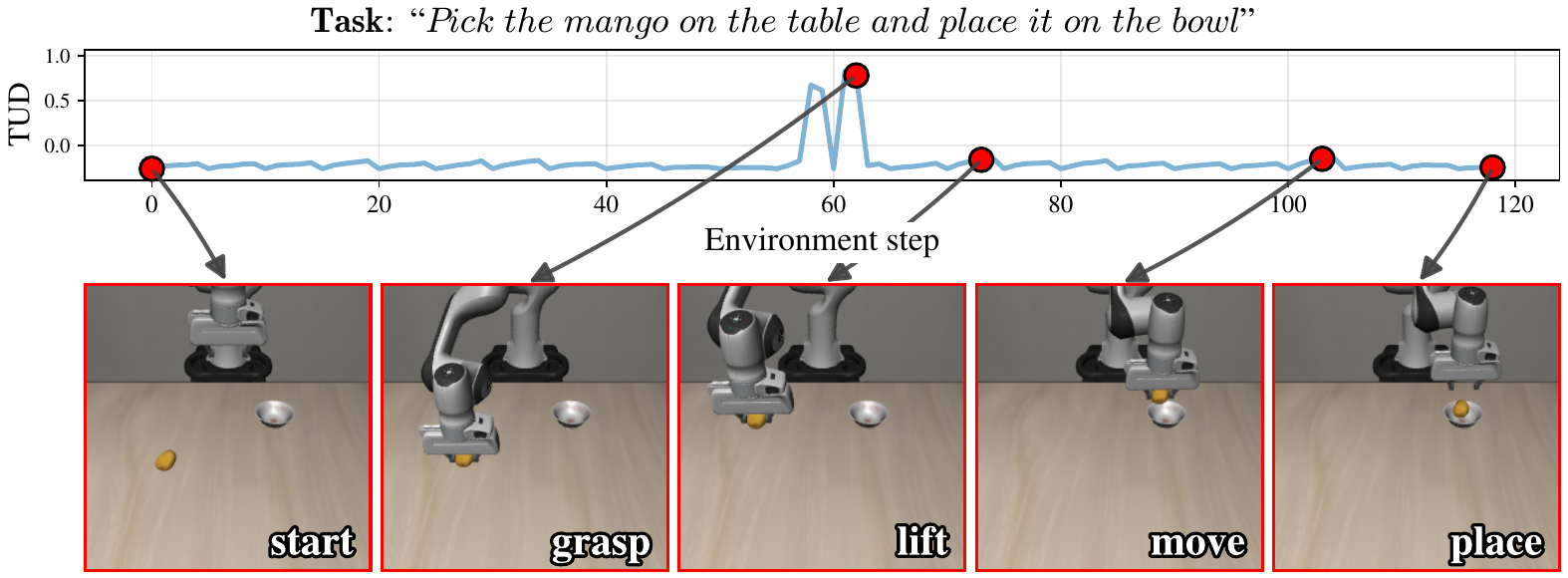}

    \caption{\textbf{Additional per-rollout $u_t$ trajectories on VLA-Arena.} From top to bottom: extrapolation (preposition combinations) failure,
    extrapolation (unseen objects) success, safety (cautious grasp) success, and distractor (static distractors) success. Same plotting convention as  Figure~\ref{fig:scenario_uncertainty}.}
    \label{fig:scenario_uncertainty_supp}
\end{figure}

\clearpage
\subsection{\texorpdfstring{CALVIN ABC$\to$D Full Comparison}{CALVIN ABC->D Full Comparison}}
\label{appendix:calvin_full}

Figure~\ref{fig:calvin_pareto} reports the CALVIN ABC$\to$D evaluation results referenced in Section~\ref{sec:exp_efficiency}. The orange curve shows our method (TUD) evaluated across a sweep of the percentile $p$. Blue and green stars denote two open-source fixed-frequency baselines: OpenHelix and Robodual. TUD reaches a longer average length than both fixed-frequency baselines while using fewer VLM calls per sub-task, suggesting that adapting the generalist call frequency also carries over to long-horizon chains of sub-tasks.

\begin{figure}[ht]
    \centering
    \includegraphics[width=0.97\textwidth]{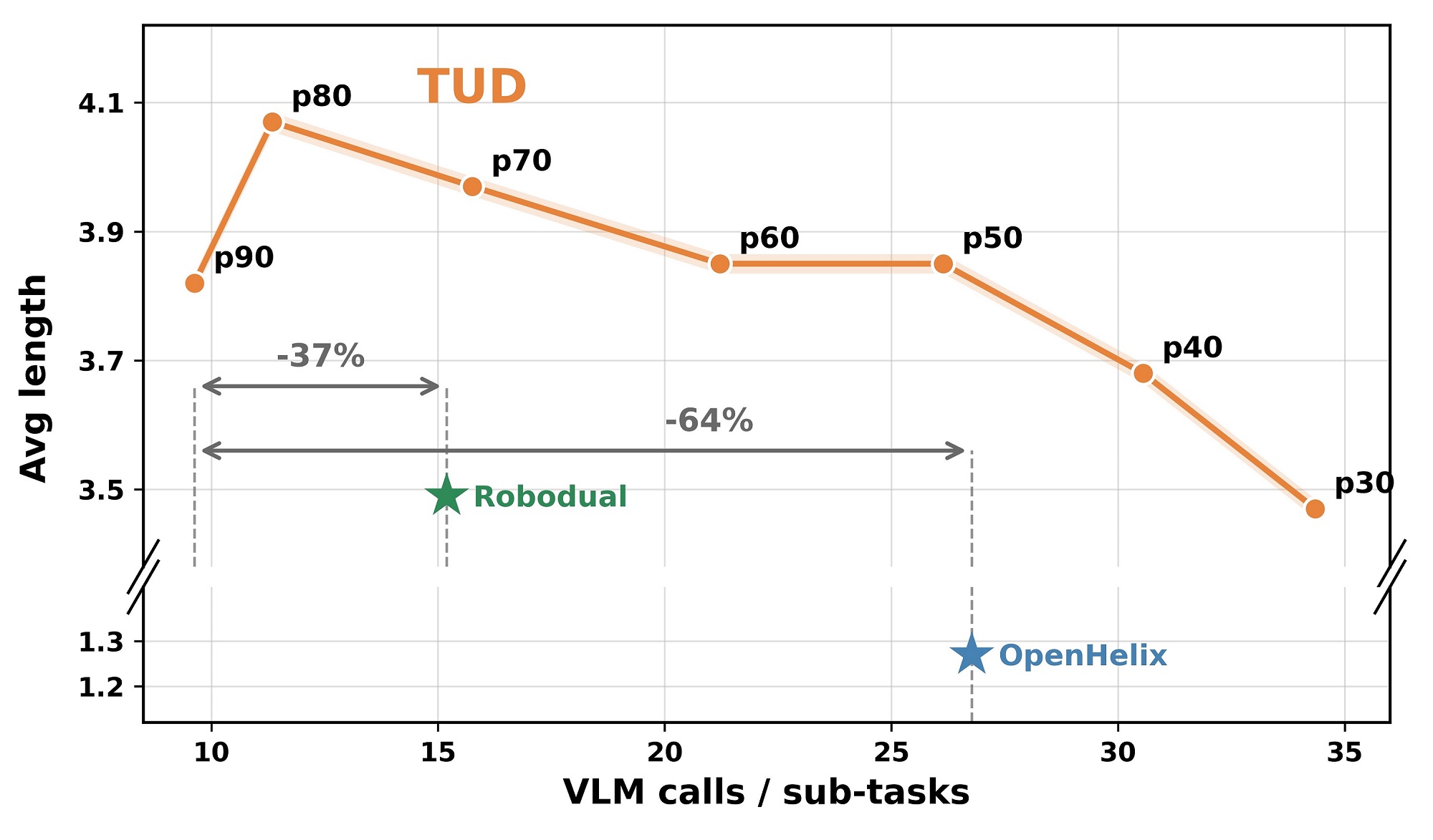}
    \caption{Cost-success trade-off on CALVIN ABC$\to$D evaluation (100 sequences, max 360 steps, each chaining 5 sub-tasks). The $x$-axis shows mean VLM calls per sub-task; the $y$-axis shows completed sub-tasks per chain (out of 5).}
    \label{fig:calvin_pareto}
\end{figure}

\section{Real Robot Experiment}
\label{appendix:real_robot}

\subsection{Setup}
\label{appendix:real_robot_setup}
\begin{wrapfigure}{r}{0.5\textwidth}
\centering
\vspace{-10pt}
\makebox[\linewidth][c]{%
    \includegraphics[width=0.45\textwidth]{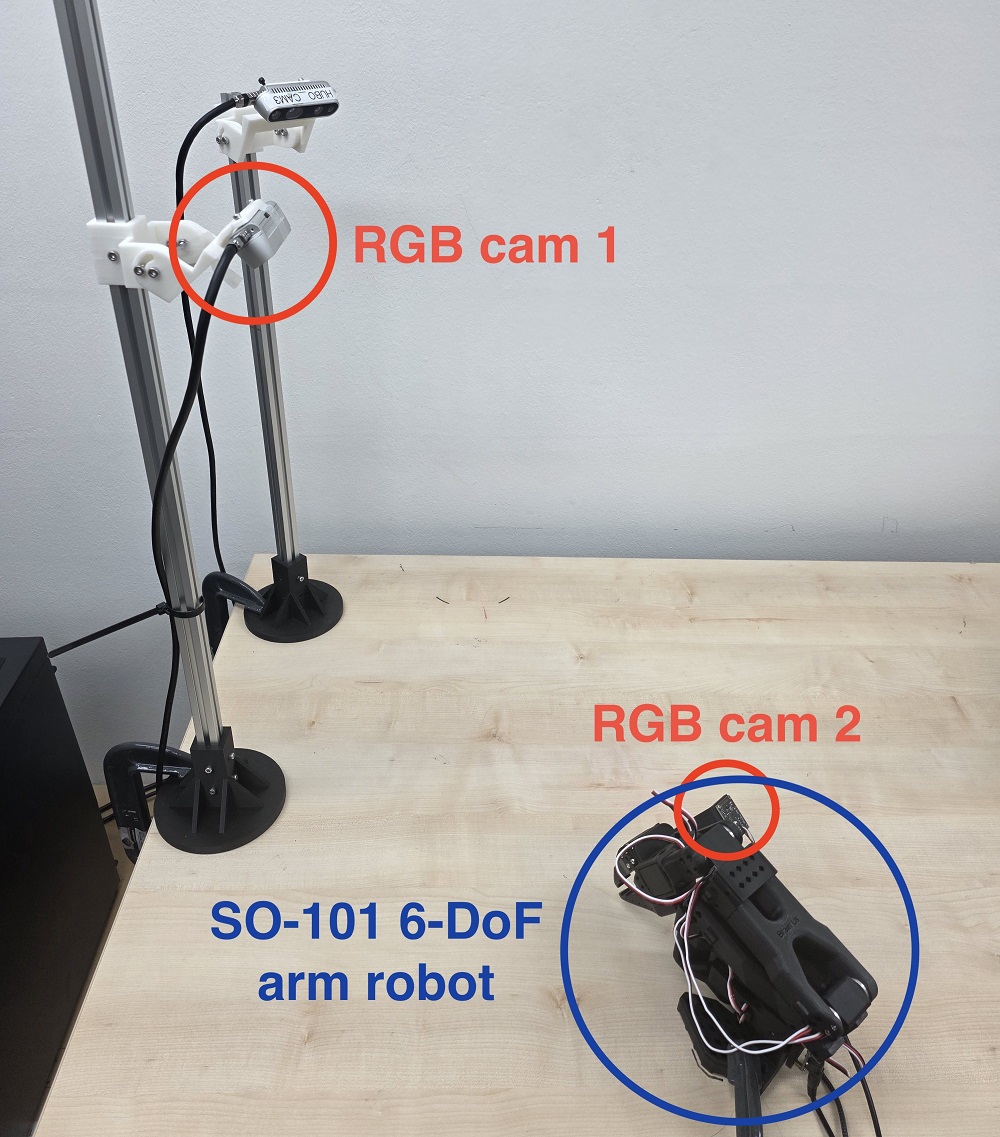}%
}
\caption{\textbf{Real robot setup.} SO-101 6-DoF arm with third-person and gripper-view RGB cameras.}
\label{fig:robot_setup}
\vspace{-10pt}
\end{wrapfigure}

As shown in Figure~\ref{fig:robot_setup}, the platform described in Section~\ref{sec:exp_setup} is an SO-101 6-DoF arm robot with two RGB cameras: a third-person view and a gripper view. Models were trained on a single H200 GPU, and inference was performed on a TITAN RTX.
We collected data, trained, and evaluated on a total of three task instructions: \textit{``Put the pen into the white bucket''}, \textit{``Close the bottom drawer''}, and \textit{``Stack the small cube on the big cube''}.
Figure~\ref{fig:test_env} shows the test environments for each task. For training data, we collected 100 tele-operated demonstrations per task while varying the object positions, and for evaluation, each method was evaluated with 10 rollouts per task across two object positions.

\begin{figure}[ht]
\centering
\caption{\textbf{Test environments.} Evaluation scenes for the three task instructions, with two object positions per task.}
\begin{tabular}{cc}
\includegraphics[width=0.42\linewidth]{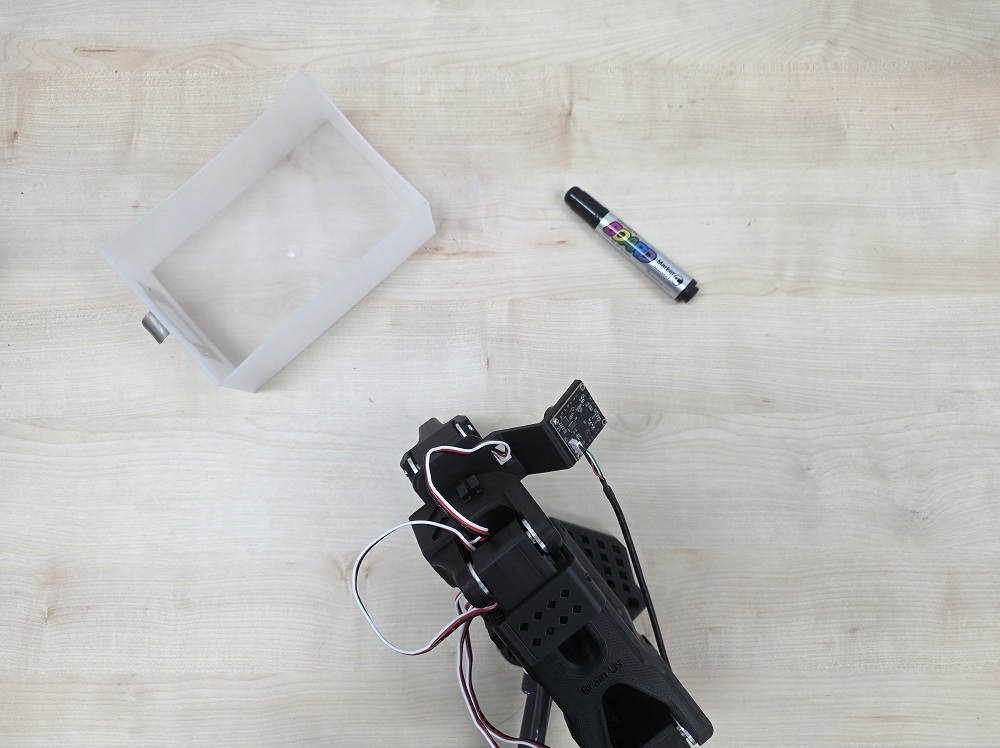} &
\includegraphics[width=0.42\linewidth]{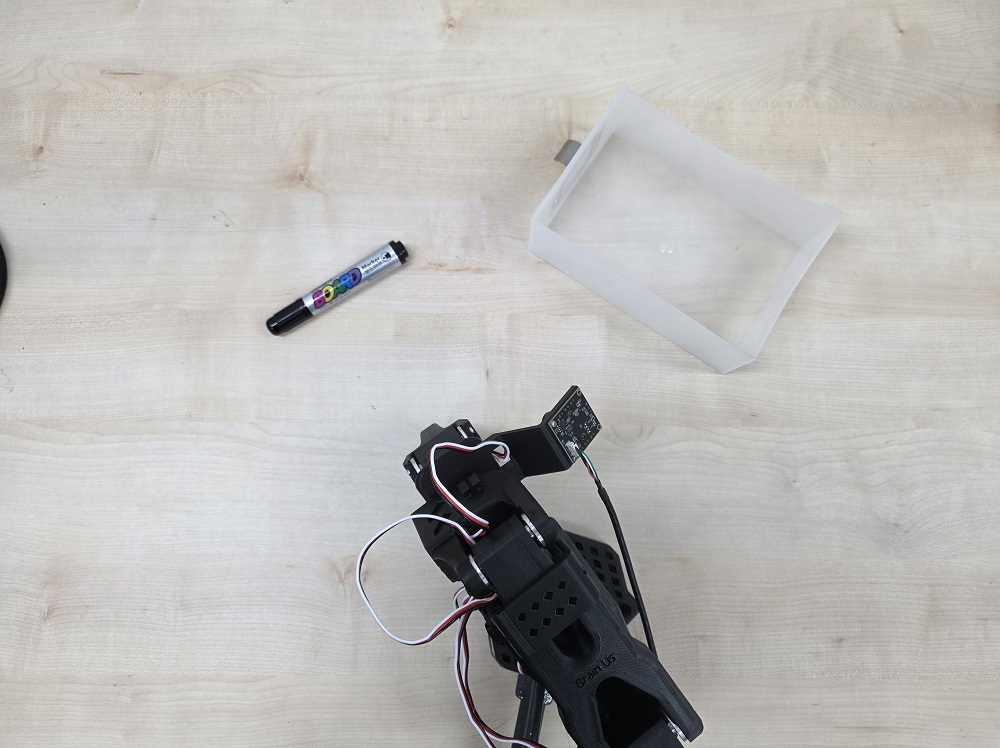} \\
\multicolumn{2}{c}{\small (a) ``Put the pen into the white bucket''} \\[6pt]
\includegraphics[width=0.42\linewidth]{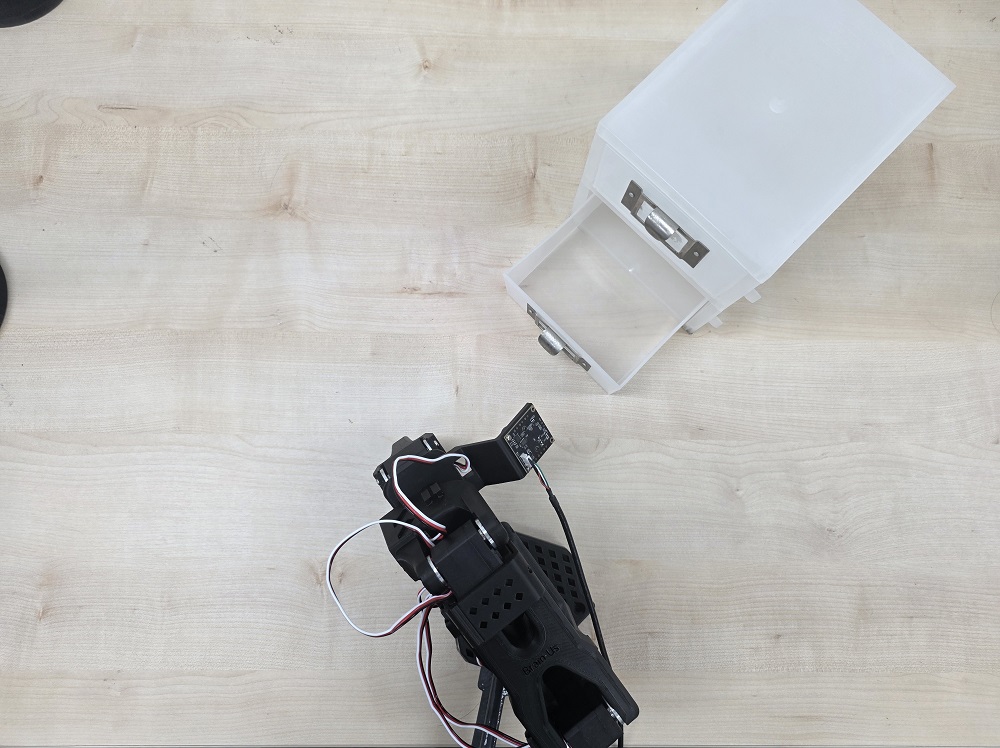} &
\includegraphics[width=0.42\linewidth]{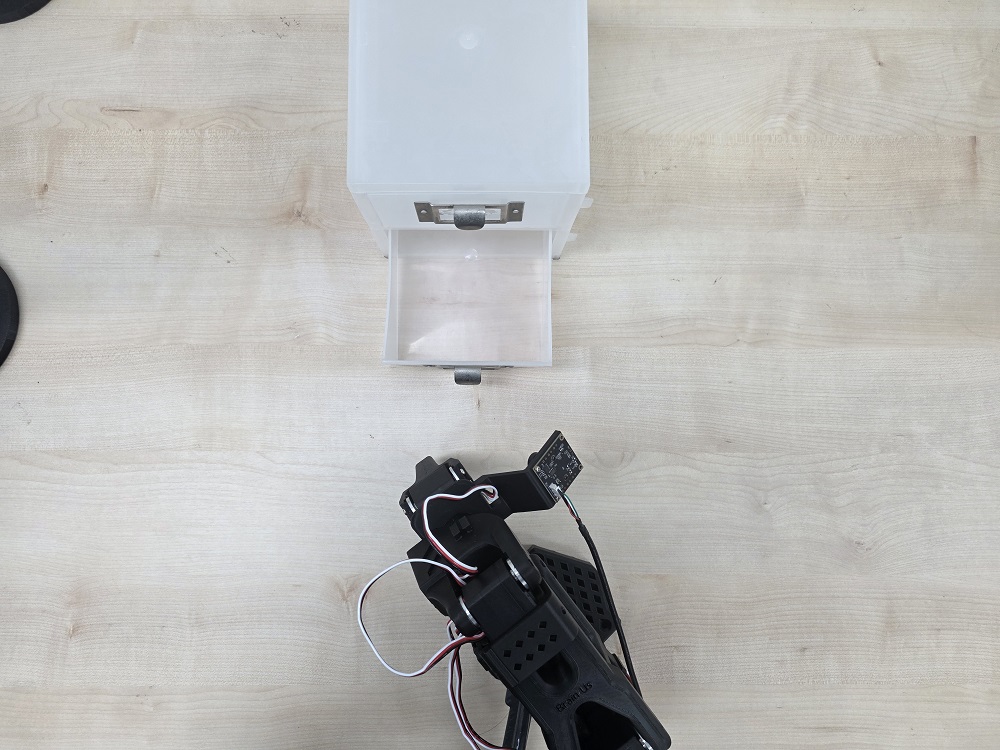} \\
\multicolumn{2}{c}{\small (b) ``Close the bottom drawer''} \\[6pt]
\includegraphics[width=0.42\linewidth]{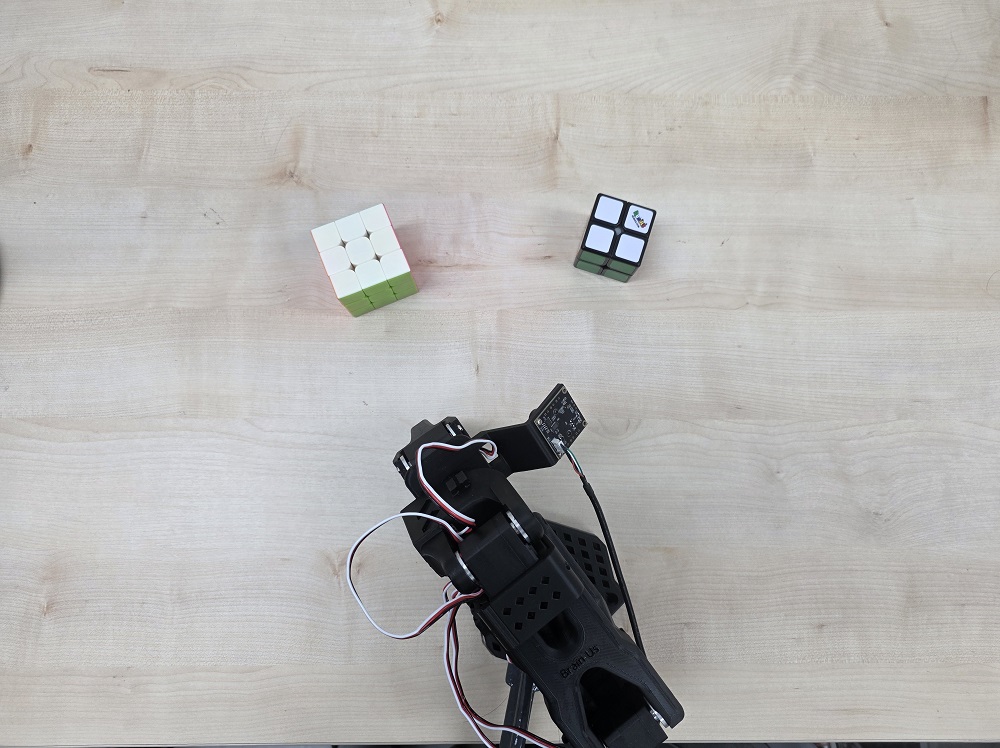} &
\includegraphics[width=0.42\linewidth]{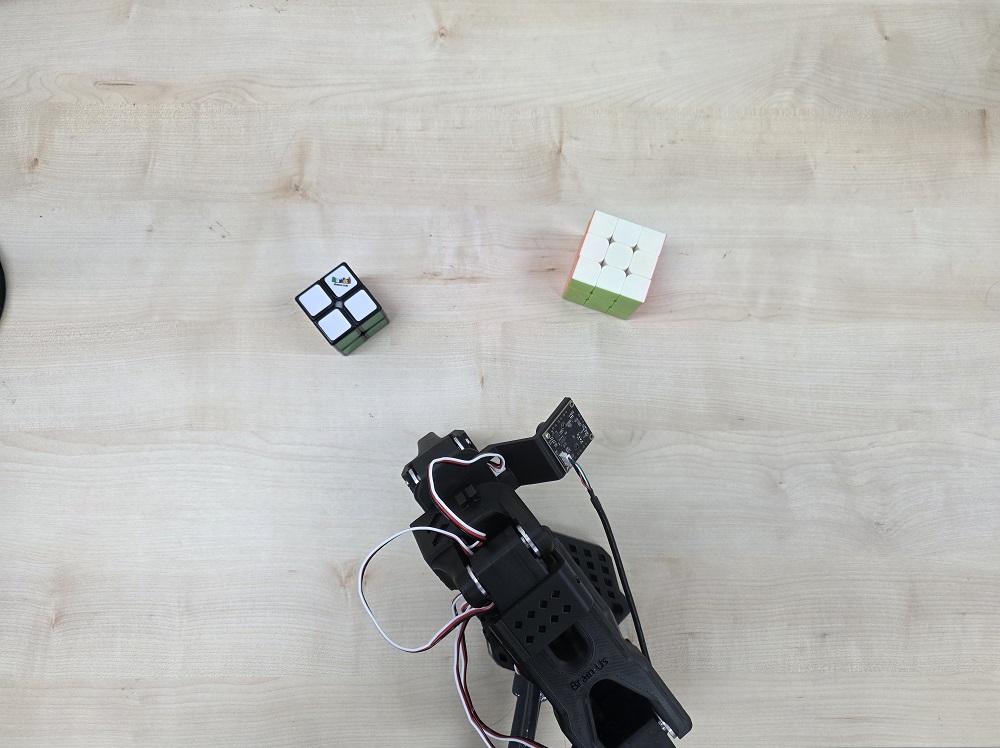} \\
\multicolumn{2}{c}{\small (c) ``Stack the small cube on the big cube''} \\
\end{tabular}
\label{fig:test_env}
\end{figure}

This section expands the real-robot evaluation of Section~\ref{sec:exp_real_robot}. Table~\ref{tab:real_robot_main_body} reports the average Success Rate, Calls/ep, and Wall-clock over the three tasks, while here we report the per-task breakdown, the end-to-end latency, and the control-loop jitter, all measured on the same rollouts.

\subsection{Per-Task Breakdown}
\label{appendix:real_robot_main}
Table~\ref{tab:real_robot_main} is the per-task performance breakdown, and it shows that the gains in performance and call count are not carried by one easy task. TUD p90 achieved the highest performance on all three task instructions, reaching $1.0$ SR on the pen task with $6.4$ calls against $42.4$ for $\pi_0$, and lifting the stack task, on which every baseline was weakest and showed a low SR, to $0.8$.

\begin{table}[ht]
\centering
\caption{\textbf{Real robot comparison per task.} Metrics are the same as in Table~\ref{tab:real_robot_main_body}. SR is the success rate; Calls and Wall are the generalist (VLM prefix) call count and the wall-clock time in seconds, both per episode. Each per-task entry is over $10$ rollouts; the \emph{mean} columns average the three tasks and reproduce the values in
Table~\ref{tab:real_robot_main_body}.}
\label{tab:real_robot_main}
\small
\resizebox{\textwidth}{!}{%
\begin{tabular}{l ccc ccc ccc ccc}
\toprule
 & \multicolumn{3}{c}{\textbf{pen}} & \multicolumn{3}{c}{\textbf{drawer}} & \multicolumn{3}{c}{\textbf{stack}} & \multicolumn{3}{c}{\textbf{mean}} \\
\cmidrule(lr){2-4} \cmidrule(lr){5-7} \cmidrule(lr){8-10} \cmidrule(lr){11-13}
Method & SR & Calls & Wall (s) & SR & Calls & Wall (s) & SR & Calls & Wall (s) & SR & Calls & Wall (s) \\
\midrule
Single-system $\pi_0$ & 0.5 & 42.4 & 69.8 & 0.5 & 47.8 & 76.7 & 0.3 & 56.8 & 92.0 & 0.43 & 49.0 & 79.5 \\
\midrule
Fixed $\Delta=9$   & 0.5 & 46.8 & 73.7 & 0.5 & 38.9 & 62.6 & 0.4 & 59.7 & 93.0  & 0.47 & 48.5 & 76.4 \\
Fixed $\Delta=25$  & 0.4 & 36.1 & 79.8 & 0.6 & 25.6 & 56.3 & 0.3 & 42.2 & 92.2  & 0.43 & 34.6 & 76.1 \\
Fixed $\Delta=50$  & 0.4 & 25.8 & 79.0 & 0.6 & 19.5 & 58.3 & 0.2 & 32.5 & 100.1 & 0.40 & 25.9 & 79.2 \\
\midrule
TUD p50 & 0.9 & 19.6 & 38.9 & 0.7 & 18.7 & 50.6 & 0.7 & 35.3 & 70.6 & 0.77 & 24.5 & 53.4 \\
TUD p70 & 0.8 & 18.6 & 43.2 & 0.6 & 22.4 & 65.0 & 0.8 & 21.3 & 52.2 & 0.73 & 20.8 & 53.5 \\
\rowcolor[HTML]{D5E8D4}
\textbf{TUD p90} & \textbf{1.0} & \textbf{6.4} & \textbf{22.6} & \textbf{0.7} & \textbf{15.2} & \textbf{48.1} & \textbf{0.8} & \textbf{14.7} & \textbf{44.9} & \textbf{0.83} & \textbf{12.1} & \textbf{38.5} \\
\bottomrule
\end{tabular}}
\end{table}

\subsection{End-to-End Latency and Control-Loop Jitter}
\label{appendix:real_robot_timing}
\begin{wraptable}{r}{0.41\textwidth}
\vspace{-1.2\baselineskip}
\centering
\caption{\textbf{Latency.} mean $\pm$ SD in ms.}
\label{tab:step_type}
\vspace{-0.4\baselineskip}
\scriptsize
\setlength{\tabcolsep}{5pt}
\renewcommand{\arraystretch}{1.05}
\begin{tabular*}{\linewidth}{@{\extracolsep{\fill}}lcc@{}}
\toprule
\textbf{Path} & \textbf{Image$\rightarrow$action} & \textbf{Cycle} \\
\midrule
Generalist & 1077 $\pm$ 109 & 1107 $\pm$ 104 \\
Specialist & 94 $\pm$ 3    & 120 $\pm$ 3 \\
\bottomrule
\end{tabular*}
\vspace{-1.3\baselineskip}
\end{wraptable}

\textbf{Latency of generalist calls and specialist forwards.} On the TITAN RTX, a generalist call takes $1.0$--$1.2$\,s, while a specialist forward takes approximately $87$\,ms. We report image-to-action delay, measured from camera-frame acquisition to transmission of the corresponding robot command, and control-cycle time, measured between consecutive commands. Table~\ref{tab:step_type} reports both measurements separately for the two execution paths.

\begin{wraptable}{r}{0.42\textwidth}
\centering
\vspace{-1.0\baselineskip}
\caption{\textbf{Control-loop jitter} (ms).}
\label{tab:timing_jitter}
\scriptsize
\setlength{\tabcolsep}{5.5pt}
\renewcommand{\arraystretch}{0.95}
\begin{tabular}{lcc}
\toprule
\textbf{Method} & \textbf{Image$\rightarrow$action} & \textbf{Cycle} \\
 & \scriptsize mean / p95 & \scriptsize mean $\pm$ SD \\
\midrule
$\pi_0$        & 70 / 1076  & 95 $\pm$ 256 \\
\midrule
Fixed $\Delta=9$    & 115 / 1180 & 136 $\pm$ 321 \\
Fixed $\Delta=25$   & 60 / 96    & 82 $\pm$ 200 \\
Fixed $\Delta=50$   & 39 / 94    & 62 $\pm$ 140 \\
\midrule
TUD p50        & 59 / 100   & 81 $\pm$ 202 \\
TUD p70        & 48 / 95    & 71 $\pm$ 175 \\
\rowcolor[HTML]{D5E8D4}
\textbf{TUD p90} & \textbf{38 / 95} & \textbf{59 $\pm$ 131} \\
\bottomrule
\end{tabular}
\vspace{-0.5\baselineskip}
\end{wraptable}
\textbf{Control-loop jitter by method.} To quantify control-loop jitter under adaptive generalist calling, we aggregate image-to-action latency and control-cycle duration over the same real-robot rollouts for each method.
As shown in Table~\ref{tab:timing_jitter}, increasing the fixed invocation interval reduces both average latency and jitter, but does so by reusing the cached generalist context for longer periods. TUD achieves a similar reduction
adaptively. In particular, TUD p90 attains image-to-action latency and control-cycle variability comparable to fixed $\Delta=50$, while achieving substantially higher task success. This suggests that adaptive invocation can retain the timing benefits of infrequent generalist calls without committing to a uniformly long reuse interval throughout the rollout.

\end{document}